\documentclass[11pt, letterpaper]{article}

\usepackage[round]{natbib}

\usepackage{fullpage}

\input{preamble.tex}

\title{\textbf{Correlated Quantization for Faster Nonconvex Distributed Optimization}}
\date{}

\author{%
	Andrei Panferov \qquad Yury Demidovich \qquad Ahmad Rammal \qquad Peter Richtárik \\
	\phantom{x}
	\\
	King Abdullah University of Science and Technology (KAUST) \\
	Thuwal, Saudi Arabia
}

\begin{document}
	
	\maketitle
	
		\begin{abstract}
		Quantization~\citep{QSGD17} is an important (stochastic) compression technique that reduces the volume of transmitted bits
		during each communication round in distributed model training. \citet{suresh2022correlated} introduce correlated quantizers and show their advantages over independent counterparts by analyzing distributed $\mathsf{SGD}$ communication complexity. We analyze the forefront distributed non-convex optimization algorithm $\mathsf{MARINA}$~\citep{gorbunov2022marina} utilizing the proposed correlated quantizers and show that it outperforms the original $\mathsf{MARINA}$ and distributed $\mathsf{SGD}$ of \citet{suresh2022correlated} with regard to the communication complexity. We significantly refine the original analysis of $\mathsf{MARINA}$ without any additional assumptions using the weighted Hessian variance~\citep{tyurin2022weightedab}, and then we expand the theoretical framework of $\mathsf{MARINA}$ to accommodate a substantially broader range of potentially correlated and biased compressors, thus dilating the applicability of the method beyond the conventional independent unbiased compressor setup. Extensive experimental results corroborate our theoretical findings.
	\end{abstract}
	\section{INTRODUCTION} 
	Modern deep neural networks consist of numerous blocks comprising diverse layers that are arranged in a hierarchical structure~\citep{lecun2015deep}. This complexity leads to a high demand for data in these networks~\citep{Attention, LMFSL20}. Moreover, it is worth noting that such models exhibit a distinct nonconvex nature~\citep{pmlr-v38-choromanska15}. Hence, there is a requirement to distribute the data among various computing resources, giving rise to the challenge of effectively orchestrating distributed~\citep{YANG2019} model training. Another incentive for adopting distributed training emerges from the Federated Learning framework~\citep{FLChallenges, KairouzFL}. In this scenario, client-owned data is not readily shared among clients. Consequently, a centralized algorithm becomes responsible for overseeing the training of multiple clients. Given that contemporary Machine Learning models have grown substantially in size, during each round of gradient descent, every client is required to transmit a dense gradient vector often comprised of millions of parameters~\citep{Li20}. This places an overwhelming strain on the communication network. Therefore, it becomes compulsory to explore techniques capable of diminishing the volume of bits transmitted over communication channels while preserving the algorithm's convergence.
	
	There exist various approaches of addressing this problem. The concept of \textit{acceleration} or \textit{momentum}~\citep{NesterovPaper, NesterovBook} in gradient-type methods has received extensive attention in conventional optimization problems. It aims to attain quicker convergence rates, thereby reducing the number of communication rounds~\citep{BeckAccel, Zhu2017KatyushaTF, LanLZ19-1, Kovalev2019DontJT,LiAccDCGD}. Deep Learning practitioners commonly rely on $\mathsf{Adam}$~\citep{KingBa15} or one of its numerous variants, which, among other techniques, also employ momentum. \textit{Local training}, which involves having each participating client perform multiple local optimization steps on their local data before engaging in communication-intensive parameter synchronization, stands out as one of the most practically valuable algorithmic components in Federated Learning model training~\citep{Povey2014ParallelTO,Moritz16,mcmahan17a,mishchenko22b,Condat2023TAMUNADA,Grudzien2023ImprovingAF}.  
	
	Driven by the necessity to create distributed stochastic gradient methods that are provably communication-efficient in nonconvex scenarios, in this paper we consider the optimization problem
	\begin{equation}\label{eq:main_problem}
		\min_{x\in\mathbb{R}^d}\left[f(x) = \frac{1}{n}\sum_{i=1}^nf_i(x)\right]
	\end{equation}
	where $n$ is the number of clients working in parallel, and $f_i : \mathbb{R}^d \to \mathbb{R}$ is a (potentially nonconvex) function representing the loss of the model parameterized by weights $x\in\mathbb{R}^d$ with respect to the training data stored on client~$i.$ We require the problem~\eqref{eq:main_problem} to be well-posed:
	\begin{assumption} \label{ass:diff} The functions $f_1,\dots,f_n: \R^d\to \R$ are differentiable. Moreover, $f$ is lower bounded, i.e., there exists $f^{\inf} \in \R$ such that $ f(x) \geq f^{\inf}$ for all $x \in \R^d$. 
	\end{assumption}
	\begin{assumption}\label{ass:Lplus}
		There exists a constant $L_{+}\geq 0$ such that $\frac{1}{n}\sum_{i=1}^n\norm{\nabla f_i(x) - \nabla f_i(y)}^2\leq L_{+}^2\norm{x-y}^2,$ for all $x,y\in\mathbb{R}^d.$ To avoid ambiguity, let $L_{+}$ be the smallest such number.
	\end{assumption}
	Assumption~\ref{ass:Lplus} is stronger than $L_{-}$-Lipschitz continuity of the gradient of $f$ (by Jensen’s inequality; also, $L_{-}\leq L_{+}$):
	\begin{assumption}
		\label{ass:local_lipschitz_constant}
		There exists a constant $L_{-}\geq0$ such that $\norm{\nabla f(x) - \nabla f(y)} \leq L_{-} \norm{x - y},$ for all $x, y \in \R^d.$
	\end{assumption}
	We are interested in finding an approximately stationary point of the nonconvex problem~\eqref{eq:main_problem}. In other words, our objective is to find a (random) vector $\widehat{x}\in\mathbb{R}^d$ such that $\Exp{\norm{\nabla f(\widehat{x})}^2}\leq\varepsilon^2$, all while minimizing the amount of communication between the $n$ clients and the server.
	
	A typical approach for solving the optimization problem~\eqref{eq:main_problem} involves employing Distributed Gradient Descent ($\mathsf{DGD}$). Starting with an initial iterate $x^0\in\mathbb{R}^d$ and a learning rate $\gamma>0,$ at each
	iteration $t,$ the server broadcasts the current iterate $x^t\in\mathbb{R}^d$ to the clients. Subsequently, each client computes its gradient $\nabla f_i(x^t)$ and sends it back to the server. Finally, the server aggregates all the gradients and utilizes them to perform the gradient descent step 
	$
	x^{t+1} = x^t - \frac{\gamma}{n}\sum_{i=1}^n\nabla f_i(x^t),
	$
	updating the iterate to $x^{t+1}.$ This process is then repeated.
	
	Although $\mathsf{DGD}$ is widely acknowledged as an optimal algorithm for attaining a stationary point with minimal iterations in smooth nonconvex problems~\citep{NesterovBook}, it also places a substantial burden on the communication network. During each communication round, $\mathsf{DGD}$ sends dense gradients to the server. As mentioned earlier, this level of communication load is deemed impractical in numerous scenarios. One approach to address this issue is to employ an unbiased compressor on the transmitted data~\citep{Seide20141bitSG, QSGD17, Lin2017DeepGC, pmlr-v70-zhang17e, Lim20183LCLA, Alistarh-EF-NIPS2018, NEURIPS2018WangAtomo}.
	\begin{definition}\label{def:unbiased_compressor}
		A (possibly randomized) mapping $\mathcal{Q}:\mathbb{R}^d\to\mathbb{R}^d$ is called an unbiased compressor if $\Exp{\mathcal{Q}(a)} = a$ and there exists a constant $\omega\geq0$ such that
		\begin{equation*}
			\Exp{\norm{\mathcal{Q}\left(a\right) - a}^2}\leq\omega\norm{a}^2,\quad\forall x\in\mathbb{R}^d.
		\end{equation*}
		If this condition is satisfied for a compressor $\mathcal{Q},$ we shall write $\mathcal{Q}\in\mathbb{U}\left(\omega\right).$
	\end{definition}
	The subsequent phase in reducing the communication burden within $\mathsf{DGD}$ involves the implementation of client-to-server communication compression. This modification of $\mathsf{DGD}$ is referred to as Distributed Compressed
	Gradient Descent ($\mathsf{DCGD}$), and it conducts iterations of the form
	\begin{equation*}
		x^{t+1} = x^t - \frac{\gamma}{n}\sum_{i=1}^n\mathcal{Q}_i^t\left(\nabla f_i(x^t)\right),
	\end{equation*}
	where $\mathcal{Q}_i^t$ is the compressor used by the client $i$ at iteration~$t.$ $\mathsf{DCGD}$ stands out as one of the simplest distributed methods that employ compression. More advanced methods include $\mathsf{DIANA}$~\citep{Diana}, $\mathsf{MARINA}$~\citep{gorbunov2022marina}.
	
	Most of the common compression techniques can be attributed to one of the two classes: \textit{sparsification} or \textit{quantization}. Sparsification~\citep{Alistarh-EF-NIPS2018} methods reduce communication by only selecting an important sparse subset of the vectors to broadcast at each step. Highly popular sparsifiers are TopK and RandK~\citep{beznosikov2020biased}. More examples can be found in the survey~\citet{demidovich2023guide}. In the present work we focus on the quantization compression technique~\citep{QSGD17}. When provided with the gradient vector at a client, we quantize each component through randomized rounding to a discrete set of values, preserving the statistical properties of the original vector. Below we define several widely used quantizers. 
	\begin{definition}\label{def:std_dithering}
		Let $1\leq q\leq+\infty,$ $a\in\mathbb{R}^d.$ Standard dithering operator~$\mathcal{D}_{sta}^{q,k}$ with $k$ levels 
		$
		0=l_k<l_{k-1}=\frac{1}{k}<\ldots<l_1=\frac{k-1}{k}<l_0=1,
		$
		is defined as follows. If $a=0,$ then $\mathcal{D}_{sta}^{q,k}=0.$ If $a\neq 0,$ let $y_i\eqdef\frac{|a_i|}{\norm{a}_q},$ for all $i\in[d].$ Fix $i,$ let $u\in\{0,1,\ldots,k-1\}$ be such that $l_{u+1}\leq y_i\leq l_u.$ Then $\left(\mathcal{D}_{sta}^{q,k}\right)_i = \norm{a}_q\times\text{sign}(a_i)\times\xi(y_i),$
		where $\xi(y_i) = l_u$ with probability $\frac{y_i-l_{u+1}}{l_u-l_{u+1}}$ or $\xi(y_i) = l_{u+1}$ otherwise.
	\end{definition}
	\begin{definition}[Natural dithering]\label{def:natural_dithering}
		Natural dithering operator~$\mathcal{D}_{nat}^{q,k}$ with $k$ levels is defined in the same way, but with $l_k=0,\;l_{k-1}=\frac{1}{2^{k-1}},\;\ldots,\;l_1=\frac{1}{2},\;l_0=\frac{1}{2^0}=1.$
	\end{definition}
	In particular, \citet{QSGD17} consider $\mathsf{QSGD}$ with independent standard dithering quantizers~$\mathcal{D}_{sta}^{2,k}$. 
	
	The major measure of the effectivness of the distibuted training method is its \textit{communication complexity}. It is the result of multiplying the number of communication rounds required to find $\widehat{x}$ by a properly defined measure of the amount of communication carried out in each round. Consistent with the standard practice in the literature, we make the assumption that client-to-server communication constitutes the primary bottleneck, and therefore, we do not include server-to-client communication in our calculations.
	
	\begin{table*}[!t]
		\centering
		\footnotesize
		\caption{Comparison of communication complexities of different distributed methods combined with different quantizers in the nonconvex regime with homogeneous clients (see Section~\ref{sec:why_correlation}), when $d\leq n.$ In the homogeneous scenario, $L_- = L_+ = L$ and $L_\pm=0$. Notation: $\Delta^0~=~f(x^0) - f^{*}.$ Abbreviations: CQ = ``Correlated Quantizers'', ISCC = ``Importance Sampling Combinatorial Compressors'', IQ = ``Independent Quantizers''.}
		\label{tbl:main}
		\begin{threeparttable}
			\begin{tabular}{ccccc}
				&&&&  \\			
				{\bf Method}  &  \bf \shortstack{Quantizer}&    {\bf \shortstack{Communication Complexity}} & \bf Correlated Compressors & \bf  \shortstack{Reference} \\			
				\hline
				$\mathsf{DCGD}$ 
				& IQ, Def.~\ref{def:multidim_independent}
				& $\mathcal{O}\left(\frac{\Delta^0dL}{\varepsilon^2}\right)$
				&\xmark
				& \citet{suresh2022correlated}\\ 
								\hline
				$\mathsf{DCGD}$ 
				& CQ, Def.~\ref{def:multidim_correlated}		
				& $\mathcal{O}\left(\frac{\Delta^0dL}{\varepsilon^2}\right)$
				&\color{green}\checkmark
				& \citet{suresh2022correlated}\\ 
								
				\hline
				$\mathsf{MARINA}$ 
				& $\mathcal{D}_{nat}^{q,k},$ Def.~\ref{def:natural_dithering} 						
				& $\mathcal{O}\left(\frac{\Delta^0L}{\varepsilon^2}\min\left\lbrace d, 1+\frac{d}{\sqrt{n}}\right\rbrace \right)$ 
				&\xmark
				& \citet{gorbunov2022marina}\\ 
				
				\hline
				$\mathsf{MARINA}$ 
				& IQ, Def.~\ref{def:multidim_independent}
				& $\mathcal{O}\left(\frac{\Delta^0L}{\varepsilon^2}\min\left\lbrace d, 1+\frac{d}{\sqrt{n}}\right\rbrace \right)$ 
				&\xmark
				& \citet{gorbunov2022marina}\\ 
				
				\hline
				\rowcolor{LightCyan}
				$\mathsf{MARINA}$ 
				& ISCC, Asm.~\ref{ass:weightedAB}
				& 	$\mathcal{O}\left( \frac{\Delta^0d}{\varepsilon^2} \min\left\lbrace L, \frac{L}{n}+\frac{\sqrt{\omega+1}L_{avg}}{\sqrt{n}} \right\rbrace\right)$
				&\xmark
				& Corollary~\ref{cor:is_comb_compr}, this work \\ 
				\hline
				\rowcolor{LightCyan}
				$\mathsf{MARINA}$ 
				& CQ, Def.~\ref{def:multidim_correlated}						
				& $\mathcal{O}\left(\frac{\Delta^0L}{\varepsilon^2}\min\left\lbrace d, 1+\frac{d}{	n}\right\rbrace \right)$
				& \color{green}\checkmark
				& Proposition~\ref{proposition:marina_indep_vs_correlated}, this work \\ 
			\end{tabular}
		\end{threeparttable}
	\end{table*}
		\begin{table*}[!t]
		\centering
		\footnotesize
		\caption{Comparison of important characteristics of different quantizers in the nonconvex zero-Hessian-variance regime and when $d\leq n:$ bits sent per client and MSE (Mean Square Error, Section~\ref{subsection:ab_mse}). Notation:  $\mathcal{D}_{sta}^{2,k}$ -- Standard Dithering, $\mathcal{D}_{sta}^{\infty,1}$ -- Ternary Quantization, $\mathcal{D}_{nat}^{q,k}$ -- Natural Dithering. Abbreviations: CQ = ``Correlated Quantizers'', ISCC = ``Importance Sampling Combinatorial Compressors'', IQ = ``Independent Quantizers''.}
		\label{tbl:newmain}
		\begin{threeparttable}
			\begin{tabular}{ccccc}
				&&&&  \\			
				\bf \shortstack{Quantizer}&    {\bf \shortstack{Bits Sent}} & {\bf \shortstack{MSE}} & \bf Correlated? & \bf  \shortstack{Reference} \\			
				\hline
				$\mathcal{D}_{sta}^{2,k},$ Def.~\ref{def:std_dithering}	
				& $\mathcal{O}\left(k(k+\sqrt{d})\right)$ & $\frac{\sqrt{d}}{nk}$
				& \xmark 	
				& \citet{QSGD17}\\ 
				\hline
				$\mathcal{D}_{sta}^{\infty,1},$ Def.~\ref{def:std_dithering}	
				& $31+d\log_23$		
				& $\frac{\sqrt{d}-1}{n}$ & \xmark			
				& \citet{terngrad}\\ 
				\hline
				$\mathcal{D}_{nat}^{q,k},$ Def.~\ref{def:natural_dithering} 						
				& $31+d\log_2(2k+1)$ 
				& $\frac{\sqrt{d}}{n2^{k-1}}$ &\xmark
				& \citet{gorbunov2022marina}\\ 
				\hline
				IQ, Def.~\ref{def:multidim_independent}
				& $32+d$ 
				& $\frac{d\norm{a}^2}{n},$~Cor.~\ref{cor:multidim_independent_mse}  &\xmark
				& \citet{gorbunov2022marina}\\ 
				\hline
				\rowcolor{LightCyan}
				CQ, Def.~\ref{def:multidim_correlated}		
				& $32+d$
				& $\frac{d\norm{a}^2}{n^2},$~Cor.~\ref{cor:multidim_mse_quant}&\color{green}\checkmark
				& \citet{suresh2022correlated}\\ 
				\hline
				\rowcolor{LightCyan}
				ISCC, Asm.~\ref{ass:weightedAB}
				& 	$\frac{\mathcal{O}\left(d\right)}{n}$
				& $\left(\frac{A}{n^2}\sum_{i=1}^n\frac{1}{w_i} - B\right)\left\|a\right\|^2,$~Asm.~\ref{ass:weightedAB} &\xmark
				& Corollary~\ref{cor:is_comb_compr}, this work \\ 
			\end{tabular}
		\end{threeparttable}
	\end{table*}
	
	The literature has introduced several distributed methods more advanced than $\mathsf{DCGD},$ often in conjunction with various quantization techniques. \citet{HorvathNat} consider $\mathsf{DCGD}$ with independent natural dithering quantizers.  \citet{Diana, DianaGeneral} study $\mathsf{DIANA}$ with arbitrary unbiased independent quantizers. \citet{terngrad} examine distributed $\mathsf{SGD}$ algorithm with $\mathcal{D}_{sta}^{\infty,1}.$ To the best of our knowledge, there exists only one paper~\citep{suresh2022correlated} that provides an analysis of unbiased correlated quantizers. However, these quantizers are integrated with the basic $\mathsf{DCGD}$ algorithm, and the optimization problem is exclusively examined in the convex setting. In Table~\ref{tbl:main} we compare communication complexities of the state-of-the-art method $\mathsf{MARINA}$ with correlated quantizers against other proposed combinations of algorithms and quantizers. In fact, DCGD with any quantizers has a communication complexity of $\mathcal{O}\left(\frac{\Delta^0dL}{\varepsilon^2}\right),$ and it is hard to see theoretical advantages of CQ. Our results are better: $\mathsf{MARINA}$ communication complexity is lower than of $\mathsf{DCGD},$ the proposed combinations with quantizers ISCC and CQ allow to reduce it even further. 	Since $L_{avg}$ can be $\sqrt{n}$ times smaller than $L_{+}$, $\mathsf{MARINA}$ with ISCC can converge up to $\sqrt{n}$ times faster than the original method. This suggests that we can develop more effective assumptions for the framework of dependent compressors. In Table~\ref{tbl:newmain} we compare the number of bits sent and MSE of different quantizers. CQ send roughly the same amount of bits as its competitors yet they have lower MSE, which allows $\mathsf{MARINA}$+CQ to achieve lower communication complexity (as shown in Table~\ref{tbl:main}).
	
	\section{CONTRIBUTIONS}
	Our main contributions can be summarized as follows.
	
	$\diamond$ We extend the analysis of the state-of-the-art distributed optimization method $\mathsf{MARINA}$~\citep{gorbunov2022marina} (see Algorithm~\ref{alg:marina}) beyond the use of independent quantizers. We rigorously demonstrate that $\mathsf{MARINA}$ achieves faster convergence when employing Correlated Quantizers (CQ) proposed by~\citet{suresh2022correlated} in the zero-Hessian-variance regime~\citep{szlendak2021permutation}. Specifically, we establish the communication complexity of $\mathsf{MARINA}$ with CQ and showcase a significant enhancement compared to $\mathsf{MARINA}$ integrated with strong baseline independent quantizers (see Table~\ref{tbl:main}; Proposition~\ref{proposition:marina_indep_vs_correlated}). Our experiments confirm the validity of our theoretical insights.
	
	$\diamond$ We compare two distributed algorithms that utilize correlated quantizers: $\mathsf{MARINA}$ and $\mathsf{DCGD}$ (see~Table~\ref{tbl:main}; Section~\ref{sec:exps_var_hess}). Our analysis reveals that in the zero-Hessian-variance regime, $\mathsf{MARINA}$ exhibits substantially lower communication complexity, making it a superior algorithm. Our experimental results corroborate the validity of our theoretical discoveries.
	
	$\diamond$ We demonstrate that CQ from~\cite{suresh2022correlated} exhibit significantly lower (by a factor of $n$) MSE compared to their independent counterparts when applied to homogeneous data (see Table~\ref{tbl:newmain}; Corollaries~\ref{cor:multidim_independent_mse} and~\ref{cor:multidim_mse_quant}). Furthermore, we provide insights into why these compressors are particularly effective when used with $\mathsf{MARINA}$ in the zero-Hessian-variance regime (see Section~\ref{section:zeroHV}).
	
	$\diamond$ We propose a new way to combine CQ with correlated sparsifiers~\citep{szlendak2021permutation}, allowing for even stronger compression (see Algorithm~\ref{alg:permk_cq}; Corollary~\ref{cor:permk_cq}).
	
	$\diamond$ We expand the scope of our findings by demonstrating through experiments that they remain applicable beyond the zero-Hessian-variance regime (see Section~\ref{section:experiments}; Appendix~\ref{sec:apx-experiments}).
	
	$\diamond$ The initial analysis of $\mathsf{MARINA}$ was conducted under the assumption of individual unbiasedness of compressors. We revise it and demonstrate that equivalent convergence results can be achieved for a much wider range of Distributed Mean Estimation algorithms. Additionally, under the weighted AB-inequality Assumption~\ref{ass:weightedAB}~\citep{tyurin2022weightedab}, we enhance the analysis of $\mathsf{MARINA}$ (see Algorithm~\ref{alg:combinatorial_marina}) by investigating its convergence guarantees (see Theorem~\ref{thm:main_result}). We propose an Importance Sampling Combinatorial Compressor which in combination with $\mathsf{MARINA}$ allows for an up to $\sqrt{n}$ times faster convergence than the original method (see Corollary~\ref{cor:is_comb_compr}; Table~\ref{tbl:main}). Our findings are corroborated by experiments (see Appendix~\ref{sec:exp-is-marina}).
	
	\section{MAIN RESULTS}
	Let us introduce an assumption for the set of compressors that is used in many results of the paper.
	\begin{assumption}[Individual Unbiasedness]
		\label{ass:individual_unbiasedness}
		The random operators $\mathcal{Q}_1,\ldots,\mathcal{Q}_n:\mathbb{R}^d\to\mathbb{R}^d$ are unbiased, i.e., $\EE{\mathcal{Q}_i(a)} = a$ for all $i \in \{1,2,\dots,n\}$ and all $a\in\mathbb{R}^d$. If these conditions are satisfied, we write $\{\mathcal{Q}_i\}_{i=1}^n \in \mathbb{U}_{\rm ind}$.
	\end{assumption}
	\subsection{AB-Inequality: Better Control of MSE}\label{subsection:ab_mse}
	Given $n$ vectors $a_1,\ldots,a_n\in\mathbb{R}^d,$ \textit{compression variance} or \textit{Mean Square Error (MSE)} associated with the set of randomized compressors $\{\mathcal{Q}_i\}_{i=1}^n$ is the quantity $\Exp{\left\|\frac{1}{n}\sum_{i=1}^{n}\mathcal{Q}_i\left(a_i\right) - \frac{1}{n}\sum_{i=1}^{n}a_i\right\|^2}.$ In their works, \citet{suresh2017meanestimation, suresh2022correlated} investigate the problem of distributed mean estimation under communication constraints and mainly focus on the task of minimizing the MSE of quantizers. In fact, compression variance naturally emerges in the analysis of~$\mathsf{MARINA}$ (see Algorithm~\ref{alg:marina}), a cutting-edge distributed algorithm designed for solving nonconvex optimization problems, and the theoretical communication complexity of this method linearly depends on the square root of the compression variance. Therefore,  it is crucial to identify compressors with low MSE when analyzing $\mathsf{MARINA}.$ Nevertheless, there exists a trade-off between MSE and communication cost. Typically, as MSE increases, compression becomes more aggressive, but concurrently, the number of communication rounds also increases.
	
	For this reason, recently, \citet{szlendak2021permutation} introduced the following tool for achieving a more precise control of compression variance.
	\begin{algorithm}[!t]
		\caption{MARINA}
		\label{alg:marina}
		\begin{algorithmic}[1]
			\State \textbf{Input:} initial point $x^0\in \R^d$, stepsize $\gamma>0$, probability ${p} \in (0, 1],$ number of iterations $T$
			\State $g^0 = \nabla f(x^0)$
			\For{$t =0,1,\dots,T-1$}
			\State Sample $c_t\sim\mathrm{Bern}(p)$
			\State Broadcast $g^t$ to all workers
			\For{$i=1,\ldots,n$ in parallel}
			\State $x^{t+1} = x^t - \gamma g^t$
			\State \begin{varwidth}[t]{\linewidth}
				$g_i^{t+1}=\nabla f_i(x^{t+1})$ if $c_t=1,$ and $g_i^{t+1}=g_i^t + \mathcal{Q}_i\left(\nabla f_i(x^{t+1}) - \nabla f_i(x^{t})\right)$ otherwise
			\end{varwidth}
			\EndFor
			\State $g^{t+1} = \frac{1}{n}\sum_{i=1}^{n}g_i^{t+1}$
			\EndFor
			\State\textbf{Output:} $\hat{x}^T$ chosen uniformly at random from $\{x^t\}_{t=0}^{T-1}$
		\end{algorithmic}
	\end{algorithm}
	\begin{assumption}[AB-inequality]
		\label{ass:ab}
		There exist constants $A,B\geq0,$ such that random operators $\mathcal{Q}_1,\ldots,\mathcal{Q}_n:\mathbb{R}^d\to\mathbb{R}^d$ satisfy the inequality
		\begin{equation*}
			\Exp{\left\|\frac{1}{n}\sum_{i=1}^{n}\mathcal{Q}_i\left(a_i\right) - \frac{1}{n}\sum_{i=1}^{n}a_i\right\|^2}
			\leq A\left(\frac{1}{n}\sum_{i=1}^{n}\left\|a_i\right\|^2\right) - B\left\|\frac{1}{n}\sum_{i=1}^{n}a_i\right\|^2,\\
		\end{equation*}
		for all $a_1,\ldots,a_n\in\mathbb{R}^d.$ If these conditions are satisfied, we write $\{\mathcal{Q}_i\}_{i=1}^n\in\mathbb{U}\left(A, B\right)$ for the set of operators.
	\end{assumption}
	Note that the MSE of the estimate $\frac{1}{n}\sum_{i=1}^{n}\mathcal{Q}_i\left(a_i\right)$ for $\frac{1}{n}\sum_{i=1}^{n}a_i$ on the right-hand side can be viewed as a variance of the sum of the compressors. The question of interest here is how correlation between the random compression operators, or their independence, can affect the MSE. The following observations were made by~\cite{szlendak2021permutation}. If compressors are unbiased (see Definition~\ref{def:unbiased_compressor}), then the AB-inequality holds without any assumption on their independence. Generally, requiring independence can lead to a significant improvement in the constant $A.$ Formally:
	\begin{proposition}\label{prop:abforunbiased}
		If, for all $i\in[n],$ $\mathcal{Q}_i\in\mathbb{U}\left(\omega_i\right)$ and $\{\mathcal{Q}_i\}_{i=1}^n \in \mathbb{U}_{\rm ind},$ then $\{\mathcal{Q}_i\}_{i=1}^n\in\mathbb{U}\left(\max_i\{\omega_i\},0\right).$ If we further assume that the compressors are independent, then $\{\mathcal{Q}_i\}_{i=1}^n\in\mathbb{U}\left(\frac{1}{n}\max_i\{\omega_i\},0\right).$
	\end{proposition}
	\subsection{Why Correlation May Help}\label{sec:why_correlation}
	It could be feasible to decrease the compression variance by introducing dependencies between the compressors. The right-hand side of Assumption~\ref{ass:ab} can be rewritten as 
	\begin{equation*}
			A\left[\left(1-\frac{B}{A}\right)\left(\frac{1}{n}\sum_{i=1}^{n}\norm{a_i}^2\right)+\frac{B}{A}\Var\left(a_1,\ldots,a_n\right)\right],
	\end{equation*}
	where $\Var\left(a_1,\ldots,a_n\right)=\frac{1}{n}\sum_{i=1}^{n}\left\|a_i - \sum_{i=1}^{n}a_i\right\|^2$ is the variance of the vectors $\{a_i\}_{i=1}^n.$ It is preferable to design compressors with $B$ as large as $A,$ since $\Var\left(a_1,\ldots,a_n\right)$ can be much smaller than $\frac{1}{n}\sum_{i=1}^{n}\norm{a_i}^2,$  This result was obtained by~\citet{szlendak2021permutation}: PermK sparsifiers introduced in this work are designed so that the sparsified vectors have zero scalar products, which enforces $A=B=1.$ Nevertheless, the approach of zeroing out scalar products does not apply to a quantization technique for the general set of vectors $\{a_i\}_{i=1}^n,$ as quantization does not inherently enforce sparsity in vectors. Instead, we demonstrate below that a thoughtfully designed dependence between the unbiased quantizers can yield an even more substantial enhancement of the constant $A$, while the constant $B$ remains comparatively smaller and equals zero. Further, we introduce a regime in which we attain theoretical improvements through correlation, provide an explanation for why this regime is more encompassing than the one involving clients that send homogeneous data (homogeneous clients regime), and clarify why $\mathsf{MARINA}$ particularly excels within it.
	
	\subsection{Zero-Hessian-Variance Regime}\label{section:zeroHV}
	The concept of Hessian variance was introduced by~\citet{szlendak2021permutation} and allowed the authors to refine the communication complexity analysis of~$\mathsf{MARINA}.$ First, let us provide a formal definition of it.
	\begin{definition}[Hessian Variance]\label{def:hessian_variance}
		Let $L_{\pm}\geq0$ be the smallest constant such that
		\begin{equation*}
			\frac{1}{n}\sum_{i=1}^n\norm{\nabla f_i(x) - \nabla f_i(y)}^2 - \norm{\nabla f(x) - \nabla f(y)}^2 \leq L_{\pm}^2\norm{x-y}^2,\quad x,y\in\mathbb{R}^d.
		\end{equation*}
		The quantity $L_{\pm}^2$ is called Hessian variance.
	\end{definition}
	Our theoretical results cover the setting when $L_{\pm}=0.$ It extends the case when clients are either homogeneous or nearly homogeneous with linear perturbations.
	\begin{proposition}\label{proposition:zero_hessian_implies_homogeneity}
		In the homogeneous clients regime Hessian variance is equal to $0.$ Moreover, if loss functions on all clients differ only by a linear term, then Hessian variance is equal to $0.$
	\end{proposition}
	The scenario in which $L_{\pm}=0$ holds is particularly advantageous for the $\mathsf{MARINA}$ algorithm. Owing to the structure of the local gradient updates (see line 8 of Algorithm~\ref{alg:marina}), the vectors $\nabla f_i(x^{t+1}) - \nabla f_i(x^t)$ that need to be compressed and transmitted from clients to the server during the communication round exhibit homogeneity in the zero-Hessian-variance regime. While achieving the zero-Hessian-variance regime in practice can be challenging, practical problems can indeed have $L_{\pm}$ values very close to zero. We delve into the theoretical properties of the correlated quantizers proposed by~\citet{suresh2022correlated} in the context of homogeneous data and illustrate their advantages over previously proposed quantizers on homogeneous data.
	\subsection{Superior Quantizers for $\mathsf{MARINA}$}
	
	We start with an introduction to baseline independent quantizers. For simplicity, we initially define them in a one-dimensional case and outline their properties. We focus on the homogeneous case where $a_i=a\in\mathbb{R}^d$ for all~$i\in[n].$
	\begin{definition}\label{def:onedim_independent}
		Suppose that, for all $i\in[n],$  $a_i=a\in[l,r],$ $l, r\in\mathbb{R}.$ Define independent randomized quantizers $\{\mathcal{Q}_i\}_{i=1}^n$ such that $\mathcal{Q}_i(a_i) = r$ with probability $\frac{a_i - l}{r - l}$ and $\mathcal{Q}_i(a_i) = l$ otherwise, $i\in[n].$
	\end{definition}
	\begin{proposition}\label{proposition:onedim_independent_mse}
		Quantizers $\{\mathcal{Q}_i\}_{i=1}^n$ from Definition~\ref{def:onedim_independent} are individually unbiased. The MSE of the quantizers $\{\mathcal{Q}_i\}_{i=1}^n$ can be bounded from above in the following way:
		\begin{equation*}
			\EE{\norm{\frac{1}{n}\sum\limits_{i=1}^n \parens{a_i- \mathcal{Q}_i(a_i)}}^2}\leq \frac{(r-l)^2}{4n}.
		\end{equation*}
	\end{proposition}
	
	Let us generalize the quantizers from Definition~\ref{def:onedim_independent} to multiple dimensions. 
	\begin{definition}\label{def:multidim_independent}
		Assume that each $a_i=a$ is a $d$-dimensional vector and that $\mathcal{Q}_i$ quantizes each coordiante independently as in Definition~\ref{def:onedim_independent} with $l=-\norm{a},$ $r=\norm{a}.$ We'll refer to them as Independent Quantization (IQ).
	\end{definition}
	\begin{corollary}\label{cor:multidim_independent_mse}
		Suppose each $a_i\in\mathbb{R}^d,$ $i\in[n].$ Then the MSE of quantizers $\{\mathcal{Q}_i\}_{i=1}^n$ of the set of vectors $\{a_i\}_{i=1}^n$ can be bounded from above in the following way:
		\begin{equation*}
			\EE{\norm{\frac{1}{n}\sum\limits_{i=1}^n \parens{a_i - \mathcal{Q}_i(a_i)}}^2}\leq \frac{d\norm{a}^2}{n}.
		\end{equation*}
	\end{corollary}
	Algorithm~\ref{alg:correlated_quantizers} provides a definition of one-dimensional correlated quantizers, that generalize their independent counterparts defined above. We aim to establish individual unbiasedness and bound the MSE of $\{\mathcal{Q}_i\}_{i=1}^n,$ defined in Algorithm~\ref{alg:correlated_quantizers}, in relation to the set of numbers $\{a_i\}_{i=1}^n,$ in the homogeneous case when $a_i=a\in[l,r],$ for all $i\in[n].$
	\begin{algorithm}[H]
            \caption{\textsc{CQ (one-dimensional variant)} \citep{suresh2022correlated}}
		\label{alg:correlated_quantizers}
            \begin{algorithmic}[1]
			\State\textbf{Input}: $a_1, a_2,\ldots, a_n, l, r\in\mathbb{R};$ $\forall i\in [n],$ $a_i \in [l, r]$
			\State Generate $\pi$, a random permutation of $\{0, 1,\ldots, n-1\}$
			\For{$i = 1 \ \text{to} \ n$}
				\State $y_i = \frac{a_i -l}{r-l}$.
				\State $U_i = \frac{\pi_i}{n} + \gamma_i$, where $\gamma_i$ has a continuous uniform distribution $\mathrm{U}[0, 1/n)$.
				\State $\mathcal{Q}_i(a_i) = (r-l) 1_{U_i < y_i}$.
                \EndFor
			\State \textbf{Output}: $\frac{1}{n} \sum^n_{i=1} \mathcal{Q}_i(a_i)$.
            \end{algorithmic}
	\end{algorithm}
	\begin{theorem}\label{thm:mse_quant}
		Suppose all the inputs $a_i=a,$ $i\in[n],$ lie in the range $[l, r].$ Then $\{\mathcal{Q}_i\}_{i=1}^n$ from Algorithm~\ref{alg:correlated_quantizers} are individually unbiased and the following upper bound on the MSE of the set of quantizers $\{\mathcal{Q}_i\}_{i=1}^n$ holds true:
		\begin{equation*}
			\EE{\norm{\frac{1}{n}\sum\limits_{i=1}^n \parens{a_i - \mathcal{Q}_i(a_i)}}^2}\leq\frac{(r - l)^2}{4n^2}.
		\end{equation*}
	\end{theorem}
	The generalization to multiple dimensions is performed in the same way as in the independent case.
	\begin{definition}\label{def:multidim_correlated}
		Assume that each $a_i=a$ is a $d$-dimensional vector and that $\mathcal{Q}_i$ quantizes each coordiante independently as in Algorithm~\ref{alg:correlated_quantizers} with $l=-\norm{a},$ $r=\norm{a}.$ We'll refer to them as Correlated Quantization (CQ).
	\end{definition}
	\begin{corollary}\label{cor:multidim_mse_quant}
		Suppose each $a_i=a\in\mathbb{R}^d,$ $i\in[n].$ Then $\{\mathcal{Q}_i\}_{i=1}^n$ are individually unbiased and the MSE of quantizers $\{\mathcal{Q}_i\}_{i=1}^n$ associated with the set of vectors $\{a_i\}_{i=1}^n$ can be bounded from above in the following way:
		\begin{equation*}
			\EE{\norm{\frac{1}{n}\sum\limits_{i=1}^n \parens{a_i - \mathcal{Q}_i(a_i)}}^2}\leq\frac{d\norm{a}^2}{n^2}.
		\end{equation*}
	\end{corollary}
	Notice that in both Corollaries~\ref{cor:multidim_independent_mse} and~\ref{cor:multidim_mse_quant}, the term $\norm{a}^2$ in the numerator can be replaced with $\frac{1}{n}\sum_{i=1}^{n}\norm{a_i}^2$. Consequently, IQ belongs to $\mathbb{U}\(\frac{d}{n},0\)$. Importantly, any client-wise independent quantization satisfying Assumption~\ref{ass:ab} will do so with $A=\frac{\widehat{\omega}}{n}$ and $B=0$, where $\widehat{\omega}$ is independent of $n$. Conversely, CQ adheres to Assumption~\ref{ass:ab} with $A=\frac{d}{n^2}$. As a result, with a fixed value of $d$, CQ has an $A$ constant that is smaller by a factor of $\mathcal{O}\(\frac{1}{n}\)$.
	
	Further, we analyze $\mathsf{MARINA}$ in the zero-Hessian-variance regime with independent and correlated quantizers.
  
	\begin{proposition}\label{proposition:marina_indep_vs_correlated}
		Let $L_{\pm}=0.$ Denote by $\cC_{\text{cor}}$ the communication complexity per client in $\mathsf{MARINA}$ with CQ (Definition~\ref{def:multidim_correlated}). Similarly, denote by $\cC_{\text{ind}}$ the communication complexity per client in $\mathsf{MARINA}$ with IQ (Definition~\ref{def:multidim_independent}). Then
        \begin{equation*}
            \frac{\mathcal{C}_{\text{ind}}}{\mathcal{C}_{\text{cor}}} = \frac{1 + \sqrt{\frac{\left(1-p\right)}{p}\frac{d}{4n} }}{1 + \sqrt{\frac{\left(1-p\right)}{p}\frac{d}{4n^2} }}.
        \end{equation*}
        That is, $\forall p \in [0,1]$, $\cC_{\text{cor}} \leq \cC_{\text{ind}}$. In particular, we show that $\cC_{\text{cor}} = \mathcal{O}\left(\frac{\Delta^0L}{\varepsilon^2}\min\left\lbrace d, 1 + \frac{d}{n} \right\rbrace\right)$ and $\cC_{\text{ind}} = \mathcal{O}\left(\frac{\Delta^0L}{\varepsilon^2}\min\left\lbrace d, 1 + \frac{d}{\sqrt{n}}\right\rbrace\right).$
	\end{proposition}
	Experiments suggest that when $d=n\gg1$, the complexity ratio is approximately $7.29$ (see Section~\ref{appendix:complexity_analysis}). For a more detailed discussion on the complexities and the conditions on the relation between $d$ and $n$ under which the ratio can reach up to $32,$ please refer to the experimental Section~\ref{section:d-n-plane}.
	
	\subsection{Combination with Sparsification}\label{subsection:quant+sparse}
	
	Combining different compression techniques often yields better results than using any single technique on its own \citep{Safaryan_2021,pmlr-v202-wang23t}. Motivated by this observation, in \citep[Section 2.4]{szlendak2021permutation}, the authors obtain general results for the composition of independent unbiased compressors and PermK sparsifiers~\citep{szlendak2021permutation}.  We design a new compression Algorithm~\ref{alg:permk_cq}, incorporating correlated sparsification in the form of PermK and CQ.
	\begin{algorithm}[H]
 		\caption{\textsc{PermK+CQ} (NEW)}
		\label{alg:permk_cq}
		\begin{algorithmic}[1]
			\State\textbf{Input}: $a_1, a_2,\ldots, a_n \in \mathbb{R}^d$, $\tau \in \mathbb{N}$
			\State Consider the $n\times d$ block diagonal matrix with $\tau$ blocks of size $\frac{n}{\tau}\times\frac{d}{\tau}$ filled with ones.
				\begin{enumerate}
					\item Randomly permute the matrix rows, then the columns to obtain a matrix $M_{i,j}$ that will indicate \\whether the client $i$ should send its $j$-th component.
					\item Zero out $a_{i,j}$ entries where $M_{i,j}=0$. Scale the remaining entries by $\tau$.
					\item Independently perform CQ (\ref{def:multidim_correlated}) within $\tau$ groups of entries that each block was mapped into.
				\end{enumerate}
			\State\textbf{Output}: sparse, quantized vectors $\{P\mathcal{Q}_i(a_i)\}_{i=1}^n$.
		\end{algorithmic}
	\end{algorithm}
	Notice, that when $\tau=n,$ we obtain PermK sparsifier. On the other hand, if $\tau=1,$ then the compressor behaves as CQ (when there is only one block, we do not perform any permutations). That is, when $1<\tau<n,$ we indeed have a compressor which combines CQ and PermK. As experiments suggest, for some values of $L_{\pm},$ PermK+CQ is better than PermK. It means, that our new compressor is more robust to the introduced noise than PermK.
	
	\begin{corollary}\label{cor:permk_cq}
		The compressors described in Algorithm~\ref{alg:permk_cq} use $32 + d/\tau$ bits per client and are individually unbiased. Moreover, if we assume that for each $i\in[n]: a_i=a$, then the mean square error of the quantizers $\{\mathcal{Q}_i\}_{i=1}^n$ can be bounded from above as:
		\begin{equation*}
			\EE{\norm{\frac{1}{n}\sum\limits_{i=1}^n \parens{a_i - \mathcal{Q}_i(a_i)}}^2}\leq\frac{d\tau^2\norm{a}^2}{n^2}.
		\end{equation*}
	\end{corollary}
	
	Comparing this algorithm with CQ in the zero-Hessian-variance regime, we find that it communicates approximately $\tau$ times less data, albeit with its variance increased by a factor of $\tau^2$.
	
	\section{EXPERIMENTS}\label{section:experiments}
	
	We compare the performance of $\mathsf{MARINA}$ when combined with Correlated Quantization (CQ, see Definition~\ref{def:multidim_correlated}), Independent Quantization (IQ, see Definition~\ref{def:multidim_independent}) and DRIVE \citep{vargaftik2021drive}. The latter serves as a robust non-correlated quantization baseline. Our primary objective is to ascertain whether our findings can be practically extended beyond the zero-Hessian-variance regime. In the plots, we depict the relationship between the total gradient norm and the volume of information communicated from clients to the server. To ensure a fair comparison of the various methods, we optimized the value of $p$ and fine-tuned the stepsizes individually for each method and task. Where applicable, the selected stepsize is shown as a multiplier of the theoretical stepsize. For details, see Appendix~\ref{sec:stepsize}.
	
	In Proposition~\ref{proposition:marina_indep_vs_correlated} we prove that in zero-Hessian-variance regime $\mathsf{MARINA}+$CQ has a lower communication complexity than $\mathsf{MARINA}+$IQ. Below in Section~\ref{section:d-n-plane} we also perform a numerical analysis to determine the ratio of communication complexities for $\mathsf{MARINA}+$CQ and $\mathsf{MARINA}+$IQ for different $d$ and $n.$ We explore whether it is possible to achieve a maximal speedup of $32.$ 
	
	\subsection{Quadratic Optimization Tasks with Various Hessian Variances $L_\pm$}\label{sec:exps_var_hess} 
	
	We produced a range of quadratic optimization tasks with varying smoothness constants (see Figure~\ref{fig:baseline_comparison}). The procedures used to generate these tasks provide us with control over it (see Appendix for details). We opted for $d = 1024$, $n = 128$, regularization $\lambda = 0.001$, and noise scale $s \in \{0, 0.5, 1.0\}.$  We can see that CQ outperforms IQ and is on par with, if not superior to, DRIVE even in tasks where $L_\pm$ substantially deviates from $0$. We also included $\mathsf{DCGD}$ as a baseline.
	
	We established our theory in the zero-Hessian-variance regime, but it becomes more challenging when $L_{\pm}\neq 0.$ In this scenario, there is no theoretical stepsize for MARINA+CQ, but a fair comparison is imperative. In the absence of theoretical guidance, a common approach is to choose optimal stepsizes for each method under consideration. We adjust the stepsizes by selecting the optimal ones as multiples of theoretical stepsizes by powers of $2$ (see Appendix~\ref{sec:opt_hyperparams} for details).

    \begin{figure}[H]
		\centering
		\vspace{.0in}
		\includegraphics[scale=0.55]{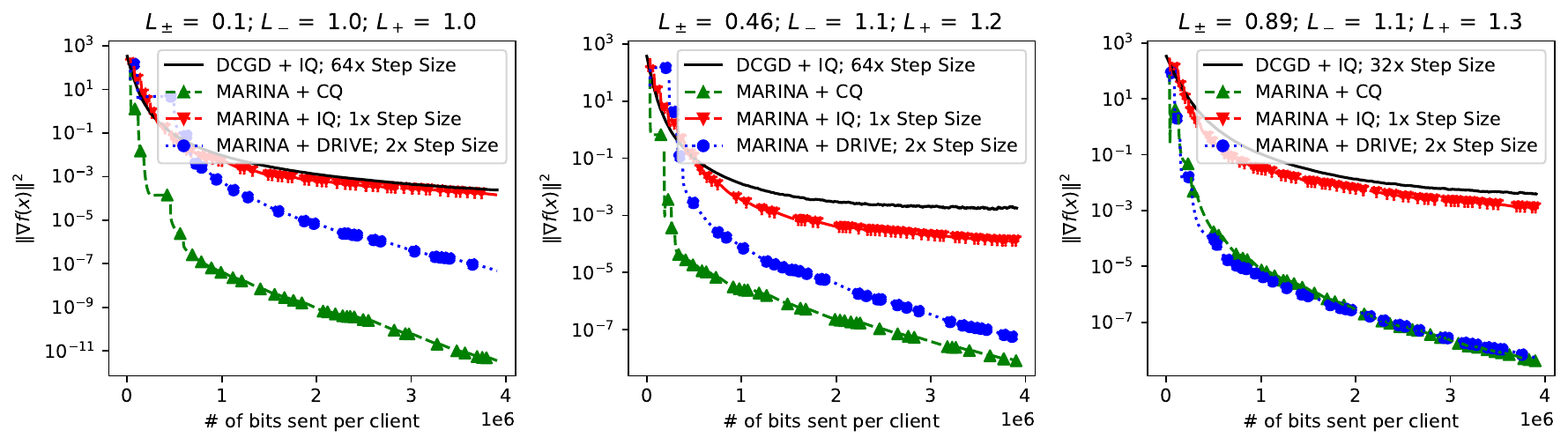}
		\vspace{.15in}
		\caption{Comparison of CQ, IQ, and DRIVE with $\mathsf{MARINA}$ on quadratic optimization tasks with diverse $L_\pm$ values}
		\label{fig:baseline_comparison}
	\end{figure}

    \subsection{Non-Convex Logistic Regression}
	
	We examine $\mathsf{MARINA}$ combined with CQ in a non-convex scenario using a logistic regression problem formulated with a non-convex regularizer:
	$$
	f(x) = \frac{1}{m}\sum\limits_{k=1}^m \log\(1 + \exp\(-y_ka_k^Tx\)\) + \lambda \sum\limits_{j=1}^d \frac{x_j^2}{1 + x_j^2},
	$$
	
	where $a_i \in \mathbb{R}^d,y_i\in \{-1, 1\}$ denote the training data, with $\lambda > 0$ as the regularization parameter. All our experiments utilized $\lambda = 0.1$.
	
	We obtained datasets from LibSVM~\citep{CC01a} and partitioned their $N$ entries into $n=d$ uniform segments. Table~\ref{tbl:libsvm} provides a summary of these datasets. Additionally, for reference, we included $\mathsf{DGD}$ (Gradient Descent) in the comparison, which can be seen as $\mathsf{MARINA}$ with no compression. We specifically choose such setting to test our approach. Notice that, mainly, it is infeasible to calculate $L_{\pm}$ for this practical problem. In general, $L_{\pm}$ should be different from zero, and we do not have a theory for $L_{\pm}\neq 0.$ The results are presented in Figure~\ref{fig:libsvm}. It can be seen that our approach is mostly dominant even in $L_{\pm}\neq 0$ case against a strong baseline $\mathsf{MARINA}+$DRIVE. We choose optimal stepsizes the same way as in Section~\ref{sec:exps_var_hess}.
	
	\begin{table}[h]
		\caption{Datasets and splitting of the data among clients} \label{tbl:libsvm}
		\begin{center}
			\begin{tabular}{lrrr}
				\textbf{Dataset}  & $n=d$ & $N$ & $\lfloor N/d \rfloor$ \\
				\hline
				\texttt{a9a} & 123 & 32,561 & 264 \\
				\texttt{madelon} & 500 &  2,000 & 4 \\
				\texttt{splice} & 60 &  1,000 & 16 \\
			\end{tabular}
		\end{center}
	\end{table}

    \begin{figure*}[h]
		\centering
		\vspace{.0in}
		\includegraphics[scale=0.54]{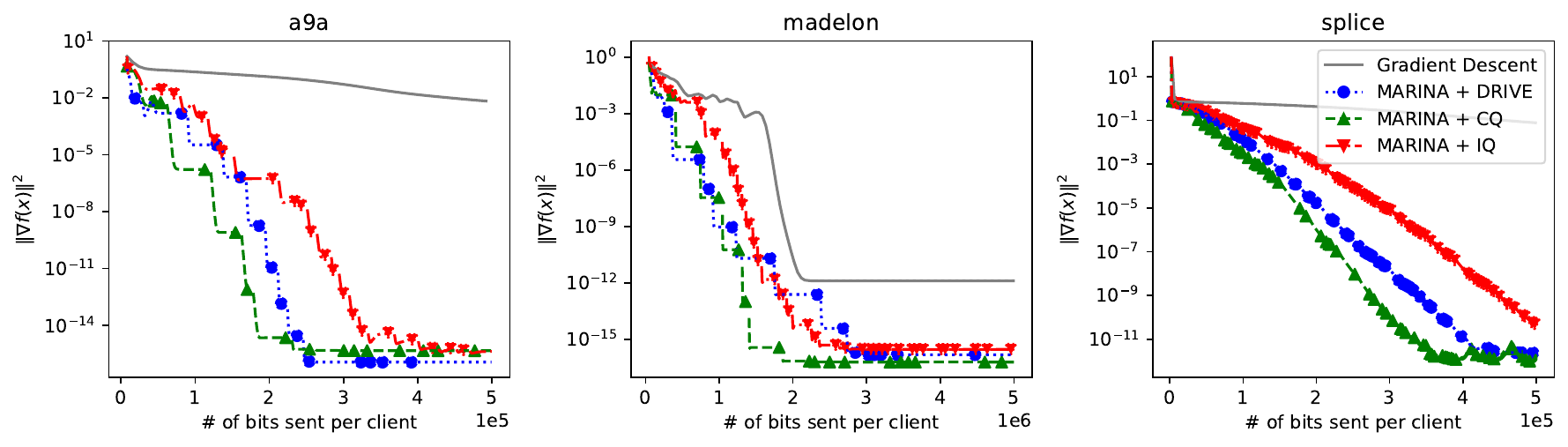}
		\vspace{.15in}
		\caption{Comparison of CQ, IQ and DRIVE with $\mathsf{MARINA}$ on LibSVM datasets. The points represent the uncompressed rounds of the algorithm}
		\label{fig:libsvm}
	\end{figure*}
	
	\subsection{Combination with PermK}
    
    \begin{minipage}{0.5\textwidth}
        As it was mentioned above in Section~\ref{subsection:quant+sparse}, combining different compression techniques, we may obtain better compressors. In Algorithm~\ref{alg:permk_cq}, we proposed a compressor that combines PermK with CQ. We empirically measure the performance of PermK+CQ on the same synthetic quadratic optimization tasks as in Section~\ref{sec:exps_var_hess}. We set $d = 1024$, $n = 3072$, $\tau=\sqrt{d}$, the regularization $\lambda = 0.001$ and $s=0.0$. We choose optimal stepsizes the same way as in Section~\ref{sec:exps_var_hess}.
    \end{minipage}\hfill
    \begin{minipage}{0.4\textwidth}
    		\vspace{.15in}
    		\centerline{\includegraphics[scale=0.57]{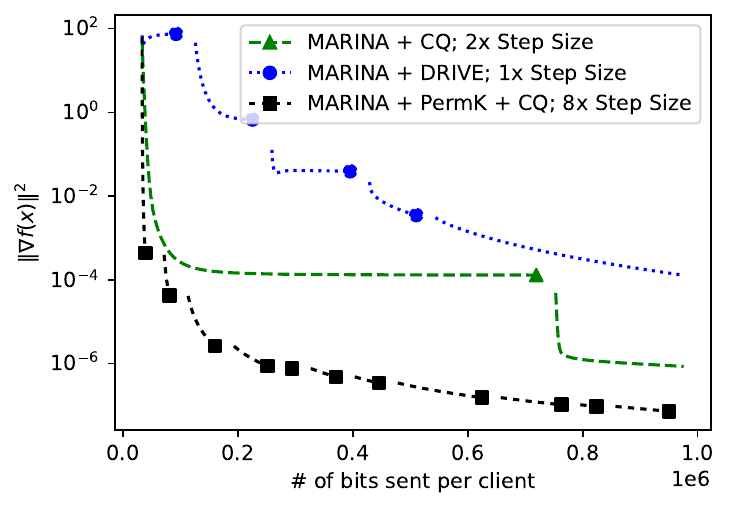}}
    		\vspace{.15in}
    		\captionof{figure}{Comparison of PermK+CQ, CQ and DRIVE with $\mathsf{MARINA}$ on quadratic optimization task with $L_\pm=0$}\label{fig:permk_cq}
    \end{minipage}
	
	\subsection{Numerical Complexity Analysis in the d-n Plane}\label{section:d-n-plane}
	Let us define an Improvement Factor as a ratio of communication complexities of $\mathsf{MARINA}$ equiped with some quantizers and $\mathsf{GD}$ (see Appendix~\ref{section:complexity_analysis}). To identify the region where CQ significantly outperforms IQ, we analyze the Improvement Factors of the $\mathsf{MARINA}$ algorithm as functions of $d$ and $n$, presuming we optimally choose the parameter $p$ of the algorithm (see \Cref{sec:opt_p}).
	
	\begin{figure}[H]
		\vspace{.0in}
		\includegraphics[scale=0.45]{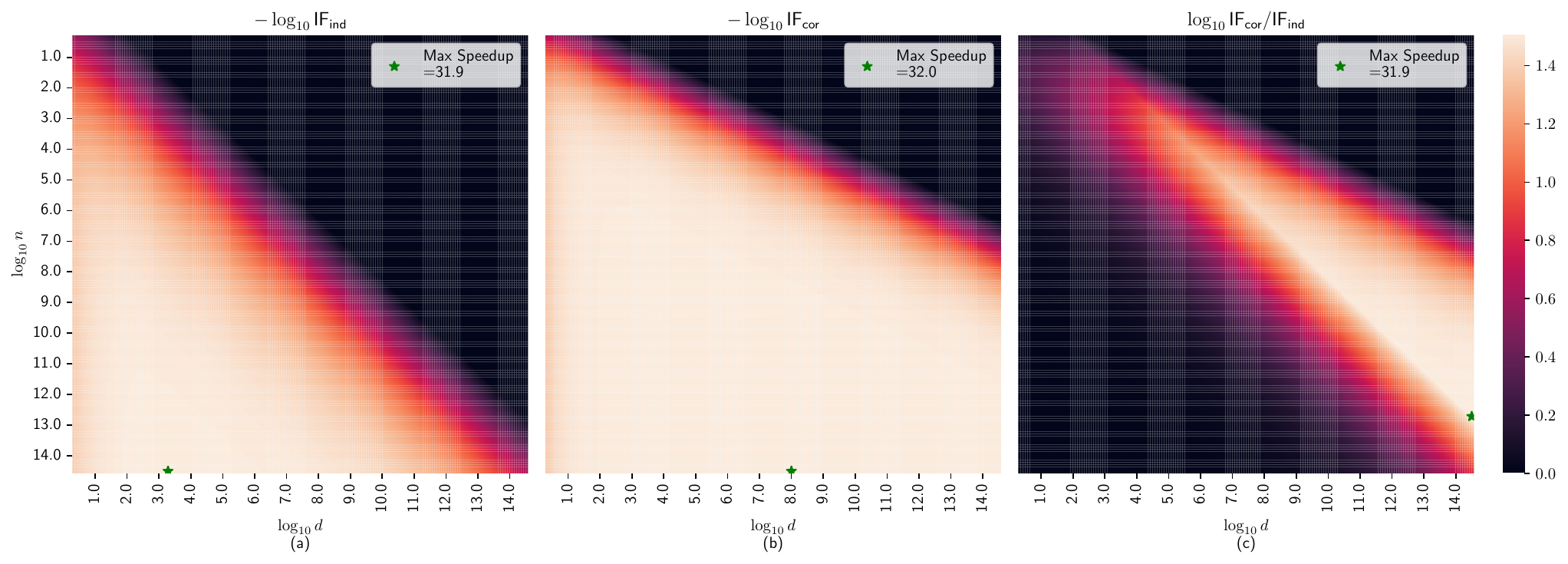}
		\vspace{.0in}
		\caption{(a)/(b): Logarithmic speedup of $\mathsf{MARINA}$ with Correlated/Uncorrelated Quantization over Gradient Descent. (c): Logarithmic speedup of $\mathsf{MARINA}+$CQ compared to $\mathsf{MARINA}+$IQ}
		\label{fig:dn_plane}
	\end{figure}
	
	From plot (a) of Figure~\ref{fig:dn_plane}, it's evident that $\mathsf{MARINA}$ with Independent Quantization defaults to Gradient Descent when $n \ll d$ and achieves the best possible speedup of x$32$ (owing to the compressor's 1-bit per coordinate behavior) when $n \gg d$. Conversely, Correlated Quantization is distinguished by $d = n^2$, as depicted in plot (b) of Figure~\ref{fig:dn_plane}. Consequently, plot (c) of Figure~\ref{fig:dn_plane} reveals that Correlated Quantization surpasses Independent Quantization by up to a factor of x$32$ when $\sqrt{d} < n < d$.
	
	\section{CONCLUSIONS AND FUTURE WORK}
	
	In the future, we would like to further explore the idea of correlated, potentially individually biased, compressors for Distributed Mean Estimation. We have shown that analyzing them in the framework of weighted AB-inequality can lead to both theoretical and practical improvements of distributed non-convex optimization algorithms. Theoretical analysis of $\mathsf{MARINA}+$CQ in case of nonhomogeneous clients remains an interesting open question for further research.
	\clearpage

	\bibliography{bibliography}
	
	\clearpage
	\appendix
	
	\onecolumn
	\section{Quantizers in Homogeneous Data Regime}
	\subsection{Analysis of Independent Quantization}
	\subsubsection{Proof of Proposition~\ref{proposition:onedim_independent_mse}}
	\begin{proof}
		Let us calculate the first moment of $\mathcal{Q}_i(x_i):$ for every $i\in[n],$ we have that 
		$$
		\Exp{\mathcal{Q}_i(x_i)} = \frac{r(x_i-l)}{r-l} + \frac{l(r-x_i)}{r-l} = x_i.
		$$
		Therefore, $\{\mathcal{Q}_i\}_{i=1}^n$ are individually unbiased. Further, let us calculate the variance of $\mathcal{Q}_i(x_i):$ for every $i$ we obtain that
		\begin{equation*}
			\begin{split}
				\Exp{\mathcal{Q}^2_i(x_i)} - \left(\Exp{\mathcal{Q}_i(x_i)}\right)^2 &= \frac{r^2(x_i-l)}{r-l} + \frac{l^2(r-x_i)}{r-l} - x_i^2\\
				& = \frac{(r-l)^2(x_i-l)(r-x_i)}{(r-l)^2}\\
				& = (r-l)^2\frac{x_i-l}{r-l}\left(1-\frac{x_i-l}{r-l}\right)\\
				& \leq \frac{(r-l)^2}{4}.
			\end{split}
		\end{equation*}
		Since the quantizers are independent and identically distributed, we obtain the following bound on the mean square error of the set $\{\mathcal{Q}_i\}_{i=1}^n:$
		\begin{equation*}
			\EE{\norm{\frac{1}{n}\sum\limits_{i=1}^n \parens{x_i - \mathcal{Q}_i(x_i)}}^2}=\frac{1}{n^2}\sum_{i=1}^{n}\Exp{\norm{\mathcal{Q}_i(x_i)-x_i}^2} \leq\frac{(r-l)^2}{4n}.
		\end{equation*} 
	\end{proof}
	\subsubsection{Proof of Corollary~\ref{cor:multidim_independent_mse}}
	\begin{proof}
		We apply the result of Theorem~\ref{thm:mse_quant} coordinate-wise and sum the variances.
	\end{proof}
	\subsection{Analysis of Correlated Quantization}
	\subsubsection{Proof of Theorem~\ref{thm:mse_quant}}
	\begin{proof}
		As shown in the proof of Theorem 2 in \cite{suresh2022correlated}, $\EE{\mathcal{Q}_i(a)} = a$ for all $a\in\mathbb{R}^d$, meaning that they are individually unbiased. We further analyze the variance in the homogeneous data regime.
		
		We first show the result when $l=0$ and $r=1$, one can obtain the final result by rescaling the quantizer operation $(r-l) \cdot \mathcal{Q}_i \left( \frac{a_i}{r-l} \right).$
		
		\begin{align*}
			\EE {\parens{ \sum^n_{i=1} a -  \sum^n_{i=1} \mathcal{Q}_i(a)}^2 }
			& =  \sum^n_{i=1}  \EE {\parens{ a -  \mathcal{Q}_i(a)}^2 } 
			+ \sum^n_{i=1} \sum_{j \neq i} \EE{\parens{  a -  \mathcal{Q}_i(a) } \parens{ a -  \mathcal{Q}_j(a) }}\\
			& = \sum^n_{i=1} a (1-a)
			+ \sum^n_{i=1} \sum_{j \neq i} \EE{\parens{  a -  \mathcal{Q}_i(a) }\parens{ a -  \mathcal{Q}_j(a) }} \\
			& =  na (1-a)
			+ \sum^n_{i=1} \sum_{j \neq i} \parens{\EE{ \mathcal{Q}_i(a)\mathcal{Q}_j(a)} -  a^2} \\
		\end{align*}
		where the second equality uses the fact that $\one_{\frac{\pi_i}{n} + \gamma_i < a}$ is a Bernoulli random variable with parameter $a$. We now calculate $ \EE{\mathcal{Q}_i(a)\mathcal{Q}_j(a)}$ for $i \neq j:$
		\begin{equation*}
			\begin{split}
				\EE{\mathcal{Q}_i(a)\mathcal{Q}_j(a)} &= \EE{\one_{\frac{\pi_i}{n} + \gamma_i < a}\one_{\frac{\pi_j}{n} + \gamma_j < a}}\\
				&= \EE{\one_{\frac{\pi_i}{n} + \gamma_i < a}\one_{\frac{\pi_j}{n} + \gamma_j < a} \parens{\one_{\pi_i>\pi_j} + \one_{\pi_j>\pi_i}}}
			\end{split}
		\end{equation*}
		Notice that since $\delta_j<\frac{1}{n}$ then $\one_{\frac{\pi_i}{n} + \gamma_i < a}\one_{\frac{\pi_j}{n} + \gamma_j < a}\one_{\pi_i>\pi_j} = \one_{\frac{\pi_i}{n} + \gamma_i < a} \one_{\pi_i>\pi_j}.$	Therefore,
		\begin{align*}
			\EE{\mathcal{Q}_i(a)\mathcal{Q}_j(a)} &= \EE{\one_{\frac{\pi_i}{n} + \gamma_i < a} \one_{\pi_i>\pi_j}} + \EE{\one_{\frac{\pi_j}{n} + \gamma_j < a} \one_{\pi_j>\pi_i}}.
		\end{align*}
		Let us calculate $\EE{\one_{\frac{\pi_i}{n} + \gamma_i < a} \one_{\pi_i>\pi_j}}$:
		\begin{equation*}
			\begin{split}
				\EE{\one_{\frac{\pi_i}{n} + \gamma_i < a} \one_{\pi_i>\pi_j}} & = \sum\limits_{k=0}^{n-1}\EE{\one_{\frac{\pi_i}{n} + \gamma_i < a} \one_{\pi_i>\pi_j}\one_{\pi_i = k}} \\
				& =\sum\limits_{k=0}^{n-1}\EE{\one_{n\gamma_i < na -k} \one_{k>\pi_j}\one_{\pi_i = k}} \\
				& \overset{\text{indep}}{=} \sum\limits_{k=0}^{n-1}\EE{\one_{n\gamma_i < nx -k}}\EE{ \one_{k>\pi_j}\one_{\pi_i = k}}\\
				& = \sum\limits_{k=0}^{n-1}\prob\parens{\one_{n\gamma_i < na -k}}\prob\parens{\brac{k>\pi_j}\cap\brac{\pi_i = k}}.\\
			\end{split}
		\end{equation*}
		
		We have that $$\prob\parens{\brac{k>\pi_j}\cap\brac{\pi_i = k}} = \prob\parens{\brac{k>\pi_j}|\pi_i = k}\prob\parens{\pi_i = k} = \frac{k}{n\parens{n-1}}.$$
		
		Therefore,
		\begin{equation*}
			\EE{\one_{\frac{\pi_i}{n} + \gamma_i < a} \one_{\pi_i>\pi_j}}
			= \frac{1}{n\parens{n-1}} \sum\limits_{k=0}^{n-1}k\prob\parens{\one_{n\gamma_i < na -k}}.
		\end{equation*}
		We have that 
		$$\prob\parens{\one_{n\gamma_i < na -k}} = 
		\begin{cases}
			1,& \text{ if } k < \floor{na} \\
			na-\floor{na},& \text{ if } k=\floor{na} \\
			0, & \text{ if } k > \floor{na}.
		\end{cases} 
		$$
		With that we get
		\begin{equation*}
			\begin{split}
				\EE{\one_{\frac{\pi_i}{n} + \gamma_i < a} \one_{\pi_i>\pi_j}} & = \frac{1}{n\parens{n-1}}\parens{ \sum\limits_{k=0}^{\floor{na}-1}k + \floor{na}\parens{na - \floor{na}} }\\
				& = \frac{1}{n\parens{n-1}}\parens{\frac{\floor{na}\parens{\floor{na}-1}}{2} + \floor{na}\parens{na - \floor{na}}} \\
				& = \frac{\floor{na}}{n\parens{n-1}}\parens{\frac{\floor{na}-1}{2} + na - \floor{na}}\\
				& = \frac{\floor{na}}{n\parens{n-1}}\parens{na - \frac{\floor{na}+1}{2}} \\
				& = \frac{na + \parens{\floor{na}-na}}{n\parens{n-1}}\parens{na - \frac{\floor{na}+1}{2}} \\
				& = \frac{1}{2}\parens{\frac{na^2}{\parens{n-1}} - \frac{a}{n-1} + \frac{na-\floor{na}}{n\parens{n-1}}\parens{\floor{na}+1-na}}.
			\end{split}
		\end{equation*}
		Let $c_a = \parens{na-\floor{na}}\parens{\floor{na}+1-na} $. We have that:
		$$\EE{\one_{\frac{\pi_i}{n} + \gamma_i < a} \one_{\pi_i>\pi_j}} = \frac{1}{2}\frac{1}{n\parens{n-1}}\parens{n^2a^2 - na + c_a}$$
		Using the symmetry between $i$ and $j$, we have that:
		$\EE{\one_{\frac{\pi_i}{n} + \gamma_i < a} \one_{\pi_i>\pi_j}} = \EE{\one_{\frac{\pi_j}{n} + \gamma_j < a} \one_{\pi_j>\pi_i}}.$ Therefore,
		\begin{align*}
			\EE{\mathcal{Q}_i(a)\mathcal{Q}_j(a)} &= \EE{\one_{\frac{\pi_i}{n} + \gamma_i < a} \one_{\pi_i>\pi_j}} + \EE{\one_{\frac{\pi_j}{n} + \gamma_j < a} \one_{\pi_j>\pi_i}}\\
			& = \frac{1}{n\parens{n-1}}\parens{n^2a^2 - na + c_a}.
		\end{align*}
		Further,
		\begin{align*}
			\EE{\mathcal{Q}_i(a)\mathcal{Q}_j(a)} - a^2 
			& = \frac{1}{n\parens{n-1}}\parens{na^2 - na + c_a}\\
			&= \frac{1}{n\parens{n-1}}\parens{-na\parens{1-a} + c_a}.
		\end{align*}
		Therefore, we have
		\begin{align*}
			\sum^n_{i=1} \sum_{j \neq i} \parens{\EE{ \mathcal{Q}_i(a)\mathcal{Q}_j(a)} -  a^2} & = -na\parens{1-a} + c_a.
		\end{align*}
		With that we get,
		\begin{align*}
			\EE {\parens{ \sum^n_{i=1} a -  \sum^n_{i=1} \mathcal{Q}_i(a)}^2 } &=  na (1-a)
			+ \sum^n_{i=1} \sum_{j \neq i} \parens{\EE{ \mathcal{Q}_i(a)\mathcal{Q}_j(a)} -  a^2} \\
			&=  na (1-a) -na\parens{1-a} + c_a\\
			& = c_a\\
			& = \parens{na-\floor{na}}\parens{\floor{na}+1-na} \\
			& = \parens{na-\floor{na}}\parens{1 - \parens{na -\floor{na}}} \\
			& \leq \frac{1}{4}.
		\end{align*}
		We get the variance by dividing this equality by $n^2.$ Therefore the variance of our quantizers is upper bounded by $\frac{1}{4n^2}.$
	\end{proof}
	\subsubsection{Proof of Corollary~\ref{cor:multidim_mse_quant}}
	\begin{proof}
		We apply the result of Theorem~\ref{thm:mse_quant} coordinate-wise and sum the variances.
	\end{proof}
	
	\subsection{Proof of Corollary~\ref{cor:permk_cq}}
	
	\begin{proof}
		Let us first show the amount of bits the compressor uses. First we notice that the image of the each block contains $n/\tau$ clients processing the same $d/\tau$ coordinates. Thus, quantized vectors require $32 + d/\tau$ bits per client, and the permutations require no extra communications, since they can be seeded.
		
		We denote Correlated Quantization by $\mathcal{Q}$, PermK by $\mathcal{P}.$
		\begin{equation*}
			\begin{split}
				\EE{\norm{\frac{1}{n}\sum_{i=1}^{n} \mathcal{Q}_i\( \mathcal{P}_i\( a_i \) \) - \frac{1}{n}\sum_{i=1}^{n} a_i }^2} &= \EE{\norm{\frac{1}{n}\sum_{i=1}^{n} \mathcal{P}_i\( a_i \) - \frac{1}{n}\sum_{i=1}^{n} a_i }^2}\\
				&+\EE{\norm{\frac{1}{n}\sum_{i=1}^{n} \mathcal{Q}_i\( \mathcal{P}_i\( a_i \) \) - \frac{1}{n}\sum_{i=1}^{n} \mathcal{P}_i\( a_i \) }^2} \\
				&= \EE{\norm{\frac{1}{n}\sum_{i=1}^{n} \mathcal{Q}_i\( \mathcal{P}_i\( a_i \) \) - \frac{1}{n}\sum_{i=1}^{n} \mathcal{P}_i\( a_i \) }^2} \\
				&\leq  \EE{\norm{\frac{1}{n}\sum_{i=1}^{n} \mathcal{Q}_i\( \mathcal{P}_i\( a \) \) - a }^2}.
			\end{split}
		\end{equation*}
		Notice that in the homogeneous case all the clients are equivalent, so the Correlated Quantization withing the image of each block can be percieved to use the same set of clients of size $n/\tau$. Let us denote the set of coordinates attributed to the image of the $k$-th block as $a_{k,j}=a_k: j\in\left[d/\tau\right]$. Since Correlated Quantization is independent coordinate-wise, the square error is additive coordinate-wise and the images of the blocks do not intersect coordinate-wise, we can freely move the sum over the blocks in and out of the norm. 
		\begin{equation*}
			\begin{split}
				\EE{\norm{\frac{1}{n}\sum_{i=1}^{n} \mathcal{Q}_i\( \mathcal{P}_i\( a \) \) - a }^2} &= \sum_{k=1}^\tau\EE{\norm{\frac{1}{n}\sum_{i=1}^{n/\tau}\tau \mathcal{Q}_i\(a_k\) - a_k}} \\
				&=\EE{\norm{\frac{\tau}{n}\sum_{i=1}^{n/\tau} \mathcal{Q}_i\(a\) - a}} \\
				&= \frac{d\norm{a}^2}{(n/\tau)^2} \\
				&= \frac{d\tau^2\norm{a}^2}{n^2}.
			\end{split}
		\end{equation*}
	\end{proof}

	\section{COMPLEXITY ANALYSIS}\label{section:complexity_analysis}
	\subsection{Proof of Proposition~\ref{proposition:marina_indep_vs_correlated}}
	\begin{proof}
		\cite{szlendak2021permutation} demonstrated that by integrating $\mathsf{MARINA}$ with a compressor that satisfies the AB-inequality and by choosing the stepsize
		\begin{align*}
			\label{theorem:AB_PL:gamma}
			\gamma \leq \left(L_- + \sqrt{\frac{\left(1-p\right)}{p}\left((A - B)L_+^2 + BL_\pm^2\right)}\right)^{-1},
		\end{align*}
		$\mathsf{MARINA}$ can identify a point $\widehat{x}^T,$ for which $\Exp{\|f(\widehat{x}^T)\|^{2}} \leq \frac{2\Delta^0}{\gamma T}.$ Notice, that in the homogeneous scenario, $L_- = L_+ = L$ and $L_\pm=0$. Thus, we can rewrite the upper bound on the stepsize 
		$$\gamma \leq \frac{1}{L}\left( 1 + \sqrt{\frac{\left(1-p\right)}{p}\left(A - B\right)}\right)^{-1}.$$
		
		Without quantization, each client will send $d$ coordinates, each composed of $32$ bits, which is $32d$ bits in total. With Correlated and Independent Quantizations, each client will send $d$ bits, plus the gradient's norm (32 bits). So $32 +d$ in total.
		In $\mathsf{MARINA}$, the expected number of bits sent per client in each step is
		$$p\parens{32d} + \parens{1-p}\parens{32+d}.$$
		To achieve an approximately stationary point $\widehat{x}$ such that $\EE{\|f(\widehat{x})\|^2} \leq \varepsilon^2$, we require 
		$$T = \frac{2\Delta_0}{\varepsilon^2} L \left( 1 + \sqrt{\frac{1-p}{p}(A - B)}\right)$$
		algorithm steps. Consequently, the overall communication complexity per client is:
		\begin{equation*}
			\begin{split}
				\cC(p) &= \parens{p\parens{32d} + \parens{1-p}\parens{32+d}}T \\
				&= \frac{2\Delta_0}{\varepsilon^2} L\parens{p\parens{32d} + \parens{1-p}\parens{32+d}} \left( 1 + \sqrt{\frac{\left(1-p\right)}{p}\left(A - B\right)}\right)\\
				&=  \underbrace{\frac{2\Delta_0}{\varepsilon^2} L \parens{32d}}_{\text{GD Rate}} \underbrace{\parens{p + \parens{1-p}\frac{32+d}{32d}} \left( 1 + \sqrt{\frac{\left(1-p\right)}{p}\left(A - B\right)}\right)}_{\text{Improvement Factor}}.
			\end{split}
		\end{equation*}
		
		\noindent\textbf{Correlated Quantizers}\\
		$\{\mathcal{Q}_i\}_{i\in[n]}\in\mathbb{U}\left(\frac{d}{n^2},0\right)$, therefore:
		
		\begin{align*}
			\cC_{\text{cor}}\parens{p} = & \frac{2\Delta_0}{\varepsilon^2} L \parens{32d}\parens{p + \parens{1-p}\frac{32+d}{32d}} \left( 1 + \sqrt{\frac{\left(1-p\right)}{p}\frac{d}{4n^2} }\right).
		\end{align*}
        Applying Lemma~12 from \citet{szlendak2021permutation}, we obtain that $\cC_{\text{cor}} = \mathcal{O}\left(\frac{\Delta^0L}{\varepsilon^2}\min\left\lbrace d, 1 + \frac{d}{n} \right\rbrace\right).$
		
		\noindent\textbf{Independent Quantizers}\\
		$\{\mathcal{Q}_i\}_{i\in[n]}\in\mathbb{U}\left(\frac{d}{4n},0\right)$, therefore:
		
		\begin{align*}
			\cC_{\text{ind}}\parens{p} = & \frac{2\Delta_0}{\varepsilon^2} L \parens{32d}\parens{p + \parens{1-p}\frac{32+d}{32d}} \left( 1 + \sqrt{\frac{\left(1-p\right)}{p}\frac{d}{4n} }\right).
		\end{align*}
            Applying Lemma~12 from \citet{szlendak2021permutation}, we obtain that $\cC_{\text{ind}} = \mathcal{O}\left(\frac{\Delta^0L}{\varepsilon^2}\min\left\lbrace d, 1 + \frac{d}{\sqrt{n}}\right\rbrace\right).$
		
		Clearly, $\forall p \in [0,1]$, $\cC_{\text{cor}}\parens{p} \leq \cC_{\text{ind}}\parens{p}$.
		
	\end{proof}
	\subsection{Extended Complexity Analysis:  the Case of n=d}\label{appendix:complexity_analysis}
	
	\noindent\textbf{Correlated Quantizers}\\
	Given $d = n \gg 1$, denoting $\frac{d}{4n^2} = \frac{1}{4n}=b$ and $\frac{32 + d}{32d} = \frac{1}{32} + \frac{1}{n}=a$, the complexity can be simplified using the fact that $b \rightarrow 0$ implies $p \rightarrow 0$:
	\begin{align*}
		\frac{-a\sqrt{b}}{2p^{3/2}} &+ \(1-a\)\(1+\sqrt{\frac{b}{p}}\) = 0,\\
		\frac{a\sqrt{b}}{2} &= \(1-a\)\(p^{3/2}+p\sqrt{b}\),\\
		p &\approx \left(\frac{a}{2(1-a)}\right)^{2/3} b^{1/3}.
	\end{align*}
	
	Substituting this into the Improvement Factor over GD, given by $\text{IF} = \frac{\cC_{\text{cor}}}{\cC_{\text{GD}}}$, we get:
	\begin{align*}
		\text{IF}_{\text{cor}} &= a + b^{1/3}\(\(a(1-a)^2/2\)^{2/3} + \(2a^2(1-a)\)^{2/3}\) + \smallO\(b^{1/3}\)= \\
		&= \frac{32 + d}{32d} + \smallO\(1\)= \frac{1}{32} + \smallO\(1\) \approx 0.03125.
	\end{align*}
	
	Thus, $\mathsf{MARINA}$ with Correlated Quantization demands approximately 0.03 times fewer bits communicated than Gradient Descent to find an $\varepsilon$-solution.
	
	\noindent\textbf{Uncorrelated Quantizers}\\
	By setting $d = n$, we denote $\frac{d}{4n} = \frac{1}{4}=b$ and $\frac{32 + d}{32d} = \frac{1}{32} + \frac{1}{d}=a$. Therefore, the problem of finding the optimal $p$ can be reduced to minimizing the function
	$$\mathcal{C}(p) = \parens{p+\parens{1-p}\(\frac{1}{32} + \smallO(1)\)}\parens{1+\sqrt{\frac{1}{4p} - \frac{1}{4}}}.$$
	Solving it numerically we get
	$$\lim_{n = d \to \infty} p \approx 0.02105,$$
	leading to
	$$\text{IF}_{\text{ind}} \approx 0.2277.$$
	
	Hence, $\mathsf{MARINA}$ with Independent Quantization requires approximately 0.23 times fewer bits communicated than Gradient Descent to find an $\varepsilon$-solution.
	
	The speedup due to correlation is then
	$$\frac{\text{IF}_{\text{ind}}}{\text{IF}_{\text{cor}}} \approx \frac{0.2277}{0.03125} \approx 7.29.$$
	
	\section{IMPROVED ANALYSIS OF $\mathsf{MARINA}$}

	\citet{szlendak2021permutation} analyzed the $\mathsf{MARINA}$ algorithm under the assumption of individual unbiasedness~\eqref{ass:individual_unbiasedness}. This algorithm employs compressed vectors to compute their average. While the assumption of individual unbiasedness guarantees the unbiasedness of the average when using independent compressors, allowing for correlated compressors at times offers a guarantee of the average's unbiasedness even without the need for individual unbiasedness. Moreover, similar to \cite{tyurin2022weightedab}, we can further refine the assumption with weights, allowing for even more sophisticated compressors.
	
	\begin{assumption}[Weighted AB-Inequality~\citep{tyurin2022weightedab}]\label{ass:weightedAB}
		Consider a random mapping $\mathcal{S}:\mathbb{R}^d\times\ldots\times\mathbb{R}^d\to\mathbb{R}^d$ to which we refer as ``combinatorial compressor'', such that, for all $a_i\in\mathbb{R}^d,\; i\in[n],$ $\Exp{\mathcal{S}\left(a_1,\ldots,a_n\right)}=\frac{1}{n}\sum_{i=1}^na_i.$ Assume that there exist $A,B\geq 0$ and weights $w_1,\ldots,w_n\in\mathbb{R}_+: \sum_{i=1}^nw_i=1$, such that, for all $a_i\in\mathbb{R}^d,\; i\in[n],$
		\begin{equation*}
			\Exp{\left\|\mathcal{S}\left(a_1,\ldots,a_n\right) - \frac{1}{n}\sum_{i=1}^{n}a_i\right\|^2}
			\leq \frac{A}{n}\sum_{i=1}^{n}\frac{1}{nw_i}\left\|a_ i\right\|^2 - B\left\|\frac{1}{n}\sum_{i=1}^{n}a_i\right\|^2.
		\end{equation*}
		The set of combinatorial compressors that satisfy this assumption is denoted by $\mathbb{S}\left(A,B,\{w_i\}_{i=1}^n\right).$
	\end{assumption}
	\begin{assumption}\label{ass:weighted_local_lipschitz_constant}
		Given a set of weights $w_1,\ldots,w_n\in\mathbb{R}_+: \sum_{i=1}^nw_i=1$, let $L_{+,w}\geq 0$ be the smallest constant such that $\frac{1}{n}\sum_{i=1}^n\frac{1}{nw_i}\norm{\nabla f_i(x) - \nabla f_i(y)}^2\leq L_{+,w}^2\norm{x-y}^2,$ for all $x,y\in\mathbb{R}^d.$
	\end{assumption}
	\begin{assumption}\label{ass:weighted_hessian_varaince}
		Given a set of weights $w_1,\ldots,w_n\in\mathbb{R}_+: \sum_{i=1}^nw_i=1$, let $L_{\pm,w}^2$ be the smallest constant such that
		\begin{equation*}
			\frac{1}{n}\sum_{i=1}^n\frac{1}{nw_i}\norm{\nabla f_i(x) - \nabla f_i(y)}^2 - \norm{\nabla f(x) - \nabla f(y)}^2
			\leq L_{\pm,w}^2\norm{x-y}^2,\quad x,y\in\mathbb{R}^d.
		\end{equation*}
		We refer to the quantity $L_{\pm,w}^2$ by the name of weighted Hessian variance.
	\end{assumption}
	
	We refine the analysis of $\mathsf{MARINA}$ under Assumption~\ref{ass:weightedAB}.
	
	\begin{theorem}\label{thm:main_result}
		Suppose that $\mathcal{S}^{t} \in \mathbb{S}\left(A,B,\{w_i\}_{i=1}^n\right),$ for all $t\in\mathbb{N},$ and that Assumptions~\ref{ass:diff},~\ref{ass:weightedAB},~\ref{ass:weighted_local_lipschitz_constant}~and~\ref{ass:weighted_hessian_varaince} hold.
		Then, for all $T > 0$ and for the stepsize $0 < \gamma \leq \left(L_- + \sqrt{\frac{1 - p}{p} \left(\left(A - B\right)L_{+,w}^2 + B L_{\pm,w}^2\right)}\right)^{-1},$ 
		the iterates produced by $\mathsf{MARINA}$ satisfy $\EE{\norm{\nabla f(\widehat{x}^T)}^2}\leq \frac{2\Delta_0}{\gamma T}$ where $\Delta_0 = f(x^0) - f^*$ and $\widehat{x}^T$ is chosen uniformly at random from $x^0, x^1, \ldots, x^{T-1}$.\\
	\end{theorem}
	This contribution fundamentally shares the same goal as the contribution with the analysis of $\mathsf{MARINA}$ with correlated quantizers: we replace the prevalent framework of individual standalone compressors found in existing literature with a framework of dependent compressors. In the first contribution, quantizers are correlated, whereas in the second contribution, the compressors are not necessarily individually unbiased, but their average is.
	
	It was originally used for analyzing sampling schemes combined with the $\mathsf{PAGE}$ method~\citep{PAGE2021} in non-distributed optimization (see~\citep{tyurin2022weightedab}). However, we employ it for compressors and improve the communication complexity of a different method, MARINA, which is used in distributed optimization. Our work demonstrates that parameters such as $L_{\pm,w}$ and $L_{+,w}$ play a pivotal role in influencing the convergence of this variance-reduced algorithm. 
	
	Furthermore, the AB-inequality, even when used independently, proves useful for simpler problems of MSE minimization. It decomposes the bound on the MSE in a natural way, allowing us to compare and analyze different sets of compressors, and it is generally tight. 
	
	\noindent \textbf{Proof of Theorem~\ref{thm:main_result}.}
	In the proof, we follow closely the analysis of \citep{gorbunov2022marina} and adapt it to utilize the power of weighted Hessian variance (Assumption~\ref{ass:weighted_hessian_varaince}) and weighted AB assumption (Assumption~\ref{ass:weightedAB}). 
	We bound the term $\Exp{\norm{g^{t+1}-\nabla f(x^{t+1})}^2}$ in a similar way to \citep{gorbunov2022marina}, but make use of the weighted AB assumption. Other steps
	are essentially identical, but refine the existing analysis through weighted Hessian variance.
	
	First, we recall the following lemmas.
	
	\begin{lemma}[\cite{PAGE2021}]
		\label{lemma:page_lemma}
		Suppose that $L_{-}$ is finite and let $x^{t+1} = x^{t} - \gamma g^{t}$. Then for any $g^{t} \in \R^d$ and $\gamma > 0$, we have
		\begin{eqnarray}
			\label{eq:page_lemma}
			f(x^{t + 1}) \leq f(x^t) - \frac{\gamma}{2}\norm{\nabla f(x^t)}^2 - \left(\frac{1}{2\gamma} - \frac{L_-}{2}\right)
			\norm{x^{t+1} - x^t}^2 + \frac{\gamma}{2}\norm{g^{t} - x^t}^2.
		\end{eqnarray}
	\end{lemma}
	
	\begin{lemma}[\cite{EF21}]
		\label{lemma:stepsize_page}
		Let $a,b>0.$ If $0 \leq \gamma \leq \frac{1}{\sqrt{a}+b},$ then $a \gamma^{2}+b \gamma \leq 1$. Moreover, the bound is tight up to the factor of 2 since $\frac{1}{\sqrt{a}+b} \leq \min \left\{\frac{1}{\sqrt{a}}, \frac{1}{b}\right\} \leq \frac{2}{\sqrt{a}+b}.$
	\end{lemma}
	
	Next, we get an upper bound of $\ExpCond{\norm{g^{t+1}-\nabla f(x^{t+1})}^2}{x^{t+1}}.$
	
	\begin{lemma}
		\label{lemma:upper_bound_variance}
		Let us consider $g^{t + 1}$ from Algorithm~\ref{alg:marina} and 
		assume, that Assumptions~\ref{ass:diff},~\ref{ass:individual_unbiasedness},~\ref{ass:weightedAB},~\ref{ass:weighted_local_lipschitz_constant}~and~\ref{ass:weighted_hessian_varaince} hold, then
		\begin{eqnarray}
			\label{eq:lemma_gradient_estimate_bound}
			\ExpCond{\norm{g^{t+1}-\nabla f(x^{t+1})}^2}{x^{t+1}}
			&\leq& (1 - p)\left(\left(A - B\right)L_{+,w}^2 + BL_{\pm,w}^2\right)\norm{x^{t+1} - x^{t}}^2 \notag\\
			&& \qquad + (1 - p)\norm{g^t - \nabla f(x^{t})}^2.
		\end{eqnarray}
	\end{lemma}
	
	\begin{proof}
		In the view of definition of $g^{t + 1}$, we get
		\begin{align*}
			&\ExpCond{\norm{g^{t+1}-\nabla f(x^{t+1})}^2}{x^{t+1}} \\
			&= (1-p) \ExpCond{\norm{g^t + \samplefunc^t\left(\{\nabla f_{i}(x^{t+1}) - \nabla f_{i}(x^t)\}_{i=1}^n\right) - \nabla f(x^{t+1})}^2}{x^{t+1}}\\
			&= (1-p)\ExpCond{\norm{\samplefunc^t\left(\{\nabla f_{i}(x^{t+1}) - \nabla f_{i}(x^t)\}_{i=1}^n\right) - \nabla f(x^{t+1}) + \nabla f(x^t)}^2}{x^{t+1}}\\
			&+ (1-p)\norm{g^t - \nabla f(x^t)}^2.\\
		\end{align*}
		In the last inequality we used unbiasedness of $\samplefunc^t.$ Next, from weighted AB inequality, we have
		\begin{align*}
			&\ExpCond{\norm{g^{t+1}-\nabla f(x^{t+1})}^2}{x^{t+1}}\\
			&\leq (1-p)\ExpCond{\norm{\samplefunc^t\left(\{\nabla f_{i}(x^{t+1}) - \nabla f_{i}(x^t)\}_{i=1}^n\right) - \nabla f(x^{t+1}) + \nabla f(x^t)}^2}{x^{t+1}}\\
			&\quad + (1-p)\norm{g^t - \nabla f(x^t)}^2\\
			&\leq (1 - p)\left(\frac{A}{n}\left(\sum\limits_{i=1}^n\frac{1}{nw_i} \norm{\nabla f_i(x^{t+1}) - \nabla f_i(x^{t})}^2\right) - B\norm{\nabla f(x^{t + 1}) - \nabla f(x^t)}^2\right)\\
			&\quad + (1 - p)\norm{g^t - \nabla f(x^{t})}^2\\
			&= (1 - p)\Bigg(\left(A - B\right)\left(\sum\limits_{i=1}^n\frac{1}{n^2w_i} \norm{\nabla f_i(x^{t+1}) - \nabla f_i(x^{t})}^2\right) \\
			&\quad + B\left(\sum\limits_{i=1}^n\frac{1}{n^2w_i} \norm{\nabla f_i(x^{t+1}) - \nabla f_i(x^{t})}^2 - \norm{\nabla f(x^{t + 1}) - \nabla f(x^t)}^2\right)\Bigg) \\
			&\quad + (1 - p)\norm{g^t - \nabla f(x^{t})}^2.\\
		\end{align*}
		Using the definition of $L_{+,w}$ and $L_{\pm,w}$, we get
		\begin{equation*}
			\begin{split}
				\Exp{\norm{g^{t+1} - \nabla f(x^{t+1})}^2} &\leq (1 - p) \left(\left(A - B\right) L_{+,w}^2 + B L_{\pm,w}^2 \right)\norm{x^{t+1} - x^t}^2\\
				& + (1 - p) \norm{g^t - \nabla f(x^{t})}^2.\\
			\end{split}
		\end{equation*}
	\end{proof}
	
	We are ready to prove Theorem~\ref{thm:main_result}. Defining
	\begin{align*}
		&\Phi^t \eqdef f(x^t) - f^{\inf} + \frac{\gamma}{2p}\norm{g^t - \nabla f(x^t)}^2,\\
		&\widehat{L}^2 \eqdef \left(A - B\right)L_{+,w}^2 + BL_{\pm,w}^2,
	\end{align*}
	and using inequalities (\ref{eq:page_lemma}) and (\ref{eq:lemma_gradient_estimate_bound}), we get
	\begin{align*}
		\Exp{\Phi^{t+1}} 
		&\leq \Exp{f(x^t) - f^{\inf} - \frac{\gamma}{2}\norm{\nabla f(x^t)}^2 - \left(\frac{1}{2\gamma} - \frac{L_-}{2}\right)\norm{x^{t+1}-x^t}^2 + \frac{\gamma}{2}\norm{g^t - \nabla f(x^t)}^2} \\
		&\quad + \frac{\gamma}{2p}\Exp{(1-p)\widehat{L}^2\norm{x^{t+1}-x^t}^2 + (1-p)\norm{g^t - \nabla f(x^t)}^2} \\
		&= \Exp{\Phi^t} - \frac{\gamma}{2}\Exp{\norm{\nabla f(x^t)}^2} \\
		&\quad + \left(\frac{\gamma(1-p)\widehat{L}^2}{2p} - \frac{1}{2\gamma} + \frac{L_{-}}{2}\right) \Exp{\norm{x^{t+1}-x^t}^2} \\
		&\leq \Exp{\Phi^t} - \frac{\gamma}{2}\Exp{\norm{\nabla f(x^t)}^2},
	\end{align*}
	where in the last inequality we use 
	$$\frac{\gamma(1-p)\widehat{L}^2}{2p} - \frac{1}{2\gamma} + \frac{L}{2} \leq 0,$$ following from the stepsize choice and Lemma \ref{lemma:stepsize_page}.
	
	Summing up inequalities $\Exp{\Phi^{t+1}} \leq \Exp{\Phi^t} - \frac{\gamma}{2}\Exp{\norm{\nabla f(x^t)}^2}$ for $t=0,1,\ldots,T-1$ and rearranging the terms, we get
	\begin{eqnarray*}
		\frac{1}{T}\sum\limits_{t=0}^{T-1}\Exp{\norm{\nabla f(x^t)}^2} &\le& \frac{2}{\gamma T}\sum\limits_{t=0}^{T-1}\left(\Exp{\Phi^t}-\Exp{\Phi^{t+1}}\right) = \frac{2\left(\Exp{\Phi^0}-\Exp{\Phi^{T}}\right)}{\gamma T} \leq \frac{2\Delta_0}{\gamma T},
	\end{eqnarray*}
	since $g^0 = \nabla f(x^0)$ and $\Phi^{T} \geq 0$. Finally, using the tower property and the definition of $\hat x^T$ from Algorithm~\ref{alg:combinatorial_marina}, we obtain the desired result.
	
	Theorem~\ref{thm:main_result} is proven.
	
	\subsection{Example: Distributed Mean Estimation}
	
	Distributed Mean Estimation algorithms are commonly assessed based on their MSE~\citep{suresh2017meanestimation,mayekar2019ratq,vargaftik2021drive,suresh2022correlated}, with it often being bounded by a factor of its input's average square norm. Naturally, such algorithms fit into Weighted AB-inequality with certain $A$, uniform weights $w_i=\frac{1}{n}$ and $B=0$, allowing for their incorporation into $\mathsf{MARINA}$, in accordance with \Cref{thm:main_result}.
	
	\subsection{Example: Importance Sampling}
	
	In this example, we consider a combinatorial compressor which is a composition of unbiased independent compressors verifying \Cref{def:unbiased_compressor}, with importance sampling~\citep{tyurin2022weightedab}.
	
	Let us recall the definition. Fix $\tau > 0.$ For all $k \in [\tau],$ we define i.i.d. random variables 
	\begin{equation*}
		\chi_{k}=\begin{cases}
			1& \text{with probability} \ q_1\\
			2&\text{with probability} \ q_2\\
			\quad &\vdots\\
			n&\text{with probability}\ q_n,
		\end{cases} 
	\end{equation*}
	where $(q_1, \dots, q_n) \in \mathcal{S}^n$ (simple simplex). 
	A sampling 
	\begin{equation*}
		\samplefunc(a_1, \dots, a_n) \eqdef \frac{1}{\tau} \sum_{k=1}^{\tau} \frac{a_{\chi_{k}}}{n q_{\chi_{k}}}
	\end{equation*} is called the Importance sampling.
	Using the result from \cite{tyurin2022weightedab},  we get:
	\begin{align*}
		\EE{\norm{\frac{1}{\tau} \sum_{k=1}^{\tau} \frac{a_{\chi_{k}}}{n q_{\chi_{k}}} - \frac{1}{n}\sum_{i=1}^n a_i}^2} = \frac{1}{\tau} \left(\frac{1}{n}\sum_{i=1}^n \frac{1}{n q_i} \norm{a_{i}}^2 - \norm{\frac{1}{n}\sum_{i=1}^n a_i}^2\right).
	\end{align*}
	In particular for $\tau=1$ we get:
	\begin{equation}
		\EE{\norm{\samplefunc(a_1, \dots, a_n) - \frac{1}{n}\sum_{i=1}^n a_i}^2} = \frac{1}{n}\sum_{i=1}^n \frac{1}{n q_i} \norm{a_{i}}^2 - \norm{\frac{1}{n}\sum_{i=1}^n a_i}^2.
		\label{eq:importance_sampling}   
	\end{equation}
	
	With this method, instead of all $n$ clients participating in the compressed rounds, only one client (selected randomly based on its ``importance'') sends their compressed vector. Next we establish whether a composition of unbiased compressors with importance sampling satisfies Assumption~\ref{ass:weightedAB}.
	
	\begin{lemma}\label{lemma:sampling}
	Let us assume that an importance sampling function $\samplefunc$ satisfies~\eqref{eq:importance_sampling} with probabilities $q_i,$ and some random compressor $\cQ$ satisfies Definition \ref{def:unbiased_compressor}. Then
		\begin{align*}
			&\EE{\norm{\samplefunc\parens{\cQ\parens{a_1}, \dots, \cQ\parens{a_n}} - \frac{1}{n}\sum_{i = 1}^n a_i}^2} \leq  \frac{\parens{1+\omega}}{n^2}\sum \limits_{i = 1}^n \frac{1}{q_i} \norm{a_i}^2  -\norm{\frac{1}{n}\sum \limits_{i = 1}^n a_i}^2 .
		\end{align*}
		
	Thus, a composition of unbiased independent compressors with importance sampling with $\tau=1$ yields a combinatorial compressor $S\in\mathbb{S}\left(\omega + 1, 1,\left\lbrace \frac{L_i}{\sum_{i=1}^nL_{i}}\right\rbrace_{i=1}^n\right).$
	\end{lemma}
	
	\begin{proof}
		Using tower property we have
		\begin{equation*}
			\begin{split}
				\EE{\norm{\samplefunc\parens{\cQ\parens{a_1}, \dots, \cQ\parens{a_n}} - \frac{1}{n}\sum_{i = 1}^n a_i}^2} & = \EE{\norm{\samplefunc\parens{\cQ\parens{a_1}, \dots, \cQ\parens{a_n}} - \frac{1}{n}\sum\limits_{i=1}^n \cQ\parens{a_i}}^2} \\
				& + \EE{\norm{\frac{1}{n}\sum\limits_{i=1}^n \cQ\parens{a_i} - \frac{1}{n}\sum\limits_{i=1}^n a_i}^2}.\\
			\end{split}
		\end{equation*}
		Let us bound the second term:
		\begin{align*}
			&	\EE{\norm{\samplefunc\parens{\cQ\parens{a_1}, \dots, \cQ\parens{a_n}} - \frac{1}{n}\sum\limits_{i=1}^n \cQ\parens{a_i}}^2}\\ & \leq \frac{1}{n}\sum \limits_{i = 1}^n \frac{1}{n q_i}\EE{\norm{\cQ\parens{a_i}}^2} - \EE{\norm{\frac{1}{n}\sum \limits_{i = 1}^n \cQ\parens{a_i}}^2} \\
			& = \frac{1}{n}\sum \limits_{i = 1}^n \frac{1}{n q_i}\EE{\norm{\cQ\parens{a_i} - a_i}^2} + \frac{1}{n}\sum \limits_{i = 1}^n \frac{1}{n q_i}\norm{a_i}^2 - \EE{\norm{\frac{1}{n}\sum \limits_{i = 1}^n \cQ\parens{a_i}}^2} \\
			& \leq  \parens{1+\omega}\frac{1}{n}\sum \limits_{i = 1}^n \frac{1}{n q_i}\norm{a_i}^2 - \EE{\norm{\frac{1}{n}\sum \limits_{i = 1}^n \cQ\parens{a_i}}^2} \\
			& = \parens{1+\omega}\frac{1}{n}\sum \limits_{i = 1}^n \frac{1}{n q_i}\norm{a_i}^2 - \EE{\norm{\frac{1}{n}\sum \limits_{i = 1}^n \cQ\parens{a_i}-a_i}^2} -\norm{\frac{1}{n}\sum \limits_{i = 1}^n a_i}^2.\\
		\end{align*}
		Therefore
		\begin{equation*}
			\begin{split}
				\EE{\norm{\samplefunc\parens{\cQ\parens{a_1}, \dots, \cQ\parens{a_n}} - \frac{1}{n}\sum_{i = 1}^n a_i}^2}  &= \EE{\norm{\samplefunc\parens{\cQ\parens{a_1}, \dots, \cQ\parens{a_n}} - \frac{1}{n}\sum\limits_{i=1}^n \cQ\parens{a_i}}^2} \\
				&+\EE{\norm{\frac{1}{n}\sum\limits_{i=1}^n \cQ\parens{a_i} - \frac{1}{n}\sum\limits_{i=1}^n a_i}^2}\\
				&\leq \frac{\parens{1+\omega}}{n^2}\sum \limits_{i = 1}^n \frac{1}{q_i} \norm{a_i}^2  -\norm{\frac{1}{n}\sum \limits_{i = 1}^n a_i}^2.\\
			\end{split}
		\end{equation*}
	\end{proof}
	
	\begin{theorem}
		\label{thm:importance_sampling_MARINA}
		Let Assumptions ~\ref{ass:diff},~\ref{ass:weighted_local_lipschitz_constant}, \ref{ass:weighted_hessian_varaince} hold. Let Assumption~\ref{ass:local_lipschitz_constant}  hold for all $f_i$ with $L_{i},$ $i\in[n].$ Given combinatorial compressors based on importance sampling $\mathcal{S}\in\mathbb{S}\left(\omega + 1, 1,\left\lbrace \frac{L_i}{\sum_{i=1}^nL_{i}}\right\rbrace_{i=1}^n\right),$ assume that $0 < \gamma \leq \parens{L_{-} + L_{avg}\sqrt{\frac{1-p}{p}\left(\omega+1\right)}}^{-1}$ where $L_{avg} = \frac{1}{n}\sum_{i=1}^n L_i$. Then for all $T \geq 0$ the iterates produced by $\mathsf{MARINA}$ satisfy $\EE{\norm{\nabla f(\widehat{x}^T)}^2}\leq \frac{2\Delta_0}{\gamma T},$ where $\Delta_0 = f(x^0) - f^*$ and $\widehat{x}^T$ is chosen uniformly at random from $x^0, x^1, \ldots, x^{T-1}$.\\ 
	\end{theorem}
	\noindent\textbf{Proof of \Cref{thm:importance_sampling_MARINA}}.
	From~\citep[Section F]{tyurin2022weightedab}, we know that by setting $q_i = \frac{L_i}{\sum_{i=1}^nL_{i}}$ for importance sampling, we obtain a sampling with $L_{+,w}^2 = L_{\pm,w}^2 = \parens{\frac{1}{n}\sum_{i=1}^nL_{i}}^2$ satisfying Assumptions \Cref{ass:weighted_local_lipschitz_constant} and \ref{ass:weighted_hessian_varaince} with $w_i = \frac{L_i}{\sum_{i=1}^nL_{i}}$. The proof is then complete by applying \Cref{thm:main_result} and \Cref{lemma:sampling}.
	
	Theorem~\ref{thm:importance_sampling_MARINA} is proven.

	\begin{corollary}\label{cor:is_comb_compr}
		Suppose assumptions of Theorem~\ref{thm:importance_sampling_MARINA} hold. Then the communication complexity of a run of $\mathsf{MARINA}$ method with importance sampling combinatorial compressors in order to reach an approximately stationary point is upper bounded by $\mathcal{O}\left( \frac{\Delta^0}{\varepsilon^2} \min\left\lbrace dL_{-}, \frac{dL_{-}}{n}+\frac{d\sqrt{\omega+1}L_{avg}}{\sqrt{n}} \right\rbrace\right)$.
	\end{corollary}
	\noindent \textbf{Proof of~\Cref{cor:is_comb_compr}.} To get an $\varepsilon$-solution it's sufficient to have $T$ iterations such that:
		$$\frac{2\Delta_0}{\gamma T} < \varepsilon^2 \iff \frac{2\Delta_0}{\gamma \varepsilon^2} < T.$$
		By taking $\gamma = \parens{L_{-} + L_{avg}\sqrt{\frac{1-p}{p}\left(\omega+1\right)}}^{-1}$, we get $T> \frac{2\Delta_0}{\varepsilon^2}\parens{L_{-} + L_{avg}\sqrt{\frac{1-p}{p}\left(\omega+1\right)}}.$
		
		Since we're doing an importance sampling with $\tau=1$, each round, with a probability $1-p$, only one client sends, on average, $\beta = \mathcal{O}(1)$ bits per coordinate.
		So the number of bits sent by round by a client is on average: $\parens{\parens{1-p}\frac{\beta}{n} + 32p}d$.
		The total complexity over all iterations is:
		$$\parens{\parens{1-p}\frac{\beta}{n} + 32p}d \times T =  \frac{2\Delta_0}{\varepsilon^2} d\parens{\parens{1-p}\frac{\beta}{n} + 32p}\parens{L_{-} + L_{avg}\sqrt{\frac{1-p}{p}\left(\omega+1\right)}}.$$
		In particular if we take $p = \frac{1}{32n}$ we get a complexity of $\mathcal{O}\left( \frac{\Delta^0}{\varepsilon^2} \parens{\frac{dL_{-}}{n}+\frac{d\sqrt{\omega+1}L_{avg}}{\sqrt{n}}} \right)$
		and if we take $p=1$ we get a communication complexity $\mathcal{O}\parens{\frac{\Delta^0}{\varepsilon^2} dL_{-}}.$\\
		Therefore, the communication complexity is upper-bounded by $\mathcal{O}\left( \frac{\Delta^0}{\varepsilon^2} \min\left\lbrace dL_{-}, \frac{dL_{-}}{n}+\frac{d\sqrt{\omega+1}L_{avg}}{\sqrt{n}} \right\rbrace\right).$
	
	\Cref{cor:is_comb_compr} is proven.\\
	
	Since $L_{avg}$ can be $\sqrt{n}$ times smaller than $L_{+}$, $\mathsf{MARINA}$ with importance sampling can converge up to $\sqrt{n}$ times faster than the original method. 
	\subsection{Experiments: Weighted $\mathsf{MARINA}$}\label{sec:exp-is-marina}
	
	We synthesized various quadratic optimization tasks with different smoothness constants $L_i$ (see Figure~\ref{fig:weighted}). We choose $d = 1024$, $n = 128$, the regularization $\lambda = 0.001$, and the noise scale $s\in\{0.0, 10.0\}$. We generated tasks so that the difference between $\max_i L_i$ and $\min_i L_i$ increases. Our experiments show that in various regimes $\mathsf{MARINA}$ combined with ISCC based on DRIVE has lower communication complexity than $\mathsf{MARINA}$ simply combined with DRIVE.
	
	\begin{figure}[h]
		\vspace{.0in}
		\centerline{\includegraphics[scale=0.57]{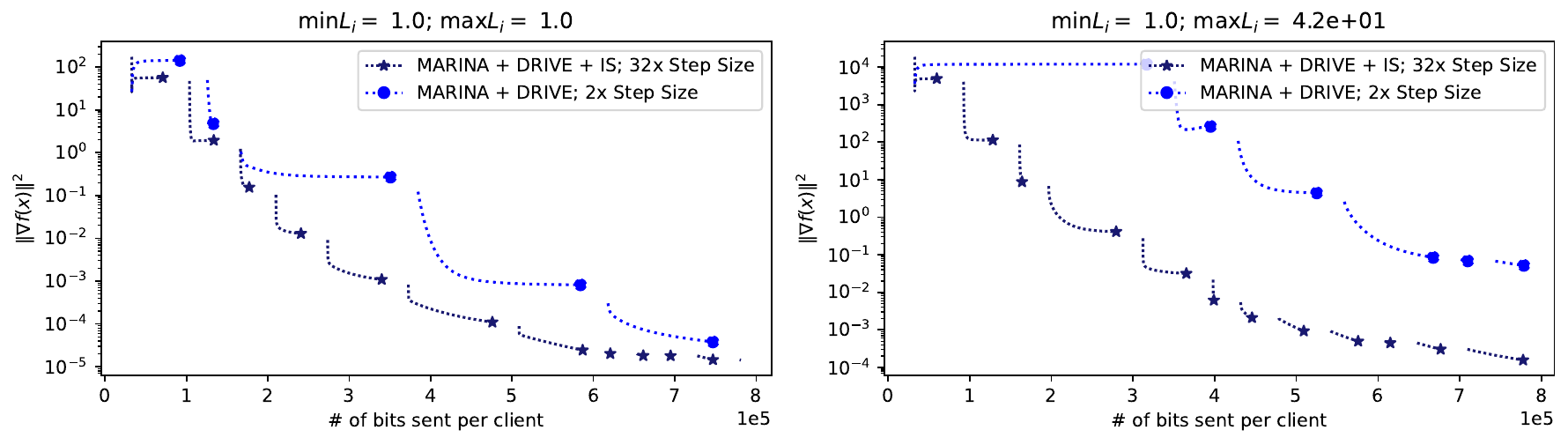}}
		\vspace{.15in}
		\caption{Comparison of DRIVE with or without Importance Sampling (IS) with $\mathsf{MARINA}$ on quadratic optimization tasks with diverse $L_\pm$ values}\label{fig:weighted}
	\end{figure}
	
	\begin{algorithm}[!t]
		\caption{MARINA with combinatorial compression}
		\label{alg:combinatorial_marina}
		\begin{algorithmic}[1]
			\State \textbf{Input:} initial point $x^0\in \R^d$, stepsize $\gamma>0$, probability ${p} \in (0, 1],$ number of iterations $T$
			\State $g^0 = \nabla f(x^0)$
			\For{$t =0,1,\dots,T-1$}
			\State Sample $c_t\sim\mathrm{Bern}(p)$
			\State $x^{t+1} = x^t - \gamma g^t$
			\State Generate a combinatorial compressor $\textbf{S}^t$
			\State \begin{varwidth}[t]{\linewidth}
				$g^{t+1}=\nabla f(x^{t+1})$ if $c_t=1,$ and $g^{t+1}=g^t + \textbf{S}^t\left(\{\nabla f_i(x^{t+1}) - \nabla f_i(x^{t})\}_{i=1}^n\right)$ otherwise
			\end{varwidth}
			\EndFor
		\end{algorithmic}
	\end{algorithm}
	
	\section{ADDITIONAL EXPERIMENTS DETAILS}\label{sec:apx-experiments}
	
	\subsection{Description of Compressors}
	
	\Cref{tbl:comp_comp} provides a comparative analysis of the compressors used in all the experiments. Notably, DRIVE transmits an equivalent number of bits per coordinate as CQ and IQ, given $d = 2^k$ for some integer $k$. To ensure a balanced comparison, we aim to choose $d$ as a power of $2$ wherever possible.
	
	\begin{table}[h]
		\caption{Comparison of the compressors used: $A$ and $B$ constants from AB-inequality~\ref{ass:ab} and the number of bits sent per client} \label{tbl:comp_comp}
		\begin{center}
			\begin{tabular}{lccc}
				\textbf{Compressor}  & $A$ & $B$ & \# of bits per client \\
				\hline
				CQ~(\Cref{def:multidim_correlated}) & $d/(4n^2)$ & $0$ & $32 + d$ \\
				IQ~(\Cref{def:multidim_independent}) & $d/(4n)$ & $0$ &  $32 + d$ \\
				DRIVE~\citep{vargaftik2021drive} & $(\pi/2-1)/n$ & $0$ &  $32 + 2^{\lceil\log_2d\rceil}$ \\
				No compression & $0$ & $0$ & $32\cdot d$
			\end{tabular}
		\end{center}
	\end{table}
	
	\subsection{Optimal Selection of Parameters}\label{sec:opt_hyperparams}
	
	\subsubsection{Identifying the Optimal Probability p}\label{sec:opt_p}
	
	Our objective is to determine the optimal probability $p_{opt}$ that reduces the communication complexity to its minimum. This is described by the equation:
	$$\mathcal{C}(p)=\frac{2\Delta_0}{\varepsilon^2}\parens{32dp+\alpha\parens{1-p}}\left(L_- + \sqrt{\frac{\left(1-p\right)}{p}\left((A - B)L_+^2 + BL_\pm^2\right)}\right),$$
	where $\alpha$ represents the expected total number of bits communicated to the server during the compressed round of $\mathsf{MARINA}$. When $B=0$, which is the case for quantization, the expression simplifies to:
	$$\mathcal{C}(p)=\frac{2\Delta_0}{\varepsilon^2}\parens{32dp+\alpha\parens{1-p}}\left(L_- + L_+\sqrt{\frac{\left(1-p\right)}{p}A}\right).$$
	
	We solve this problem numerically. This takes into account each problem's $L_+$, $L_-$, and each compressor's $\alpha$, $A$ to obtain $p_{opt}$. Specifically for CQ and its variants where $L_\pm \ne 0$, $A$ isn't explicitly defined and we extrapolate the equations from the zero-Hessian-variance regime.

	\subsubsection{Optimization of the Step Size}
	
	Having determined the value of $p$, we proceed to increase the step size. We increment the step size in multiples of $2$ ($2, 4, 8$, etc.) of the theoretically optimal step size. Our aim is to identify the step size that ensures the algorithm's best performance at $4\cdot10^6$ bits communicated from each client to the server. That number was chosen as sufficiently large to demonstrate relative convergence between different algorithms. The convergence plots, as well as details about the selected optimal step sizes can be found in Figure~\ref{fig:baseline_comparison_full}.

    \begin{figure}[H]
		\centering
		\vspace{.0in}
		\includegraphics[scale=0.55]{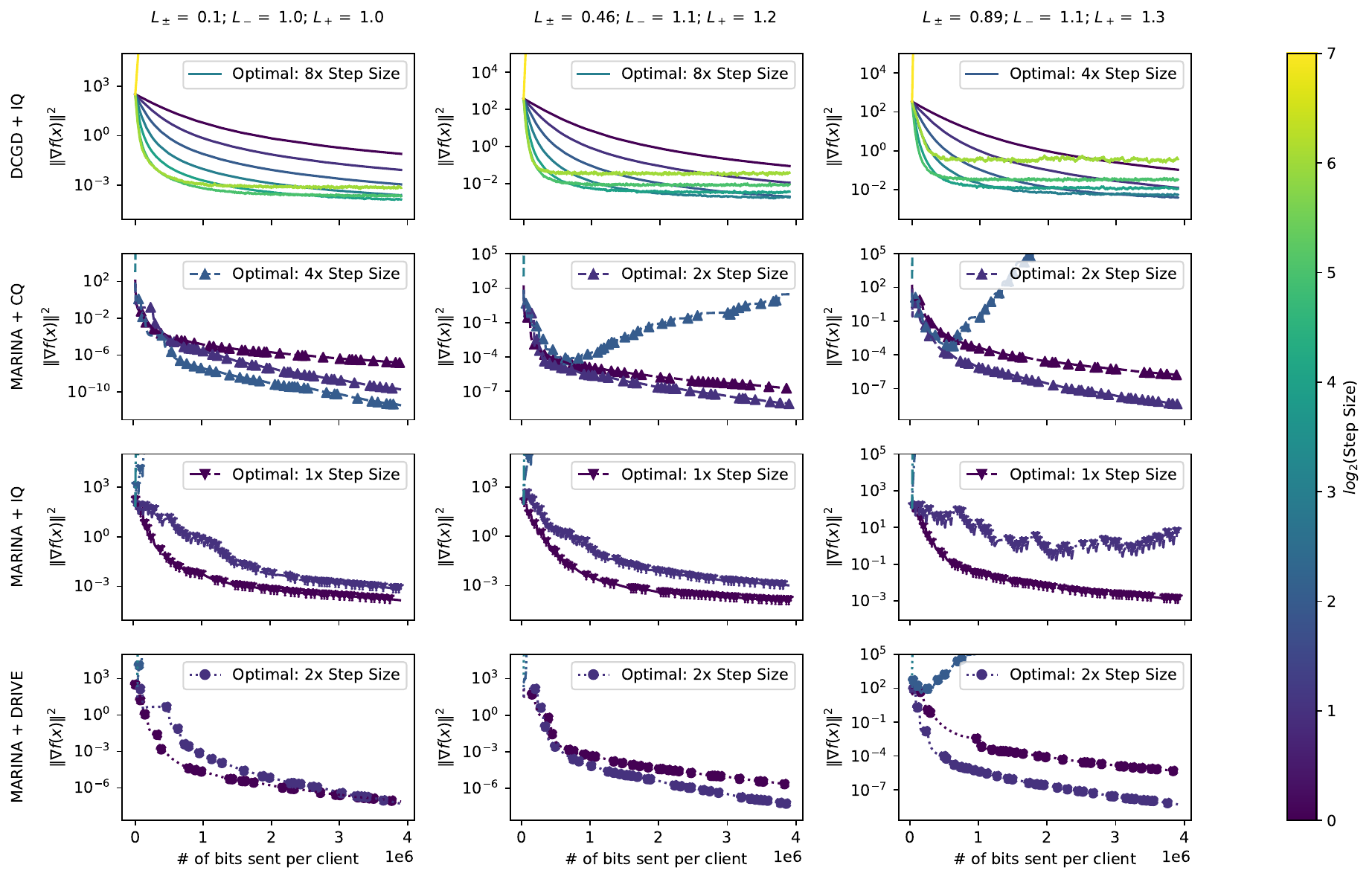}
		\vspace{.15in}
		\caption{Convergence of CQ, IQ, and DRIVE with $\mathsf{MARINA}$ with different step sizes on quadratic optimization tasks with diverse $L_\pm$ values}
		\label{fig:baseline_comparison_full}
	\end{figure}

	\subsection{Quadratic Optimization Tasks Generation}\label{sec:stepsize}
	
	Similar to \cite{tyurin2022weightedab}, we provide the algorithms used to generate artificial quadratic optimization tasks. \Cref{alg:l_pm_generation} and \Cref{alg:l_i_generation} allow us to control the smoothness constants $L_\pm$ and $L_i$, respectively, via the noise scales.
	
	\begin{algorithm}[H]
		\begin{center}
			\begin{minipage}{\textwidth}
				\noindent\textbf{Input}: number nodes $n$, dimension $d$, regularizer $\lambda$, and noise scale $s$.\newline
				\noindent For $i = 1 \ \text{to} \ n$:
				\begin{enumerate}
					\item Generate random noises $\nu_i^s=1 + s\xi_i^s$ and $\nu_i^b=s\xi_i^b$, i.i.d $\xi_i^s,\xi_i^b \sim \text{NormalDistribution}(0,1)$.
					\item Take vector $b_i = \frac{\nu_i^s}{4}\(-1+\nu_i^b, 0, \ldots,0\)\in \mathbb{R}^d$.
					\item Take the initial tridiagonal matrix
					$$
					\mathbf{A}_{i} = \frac{\nu_i^s}{4}
					\begin{bmatrix}
						2 &-1 & &0\\
						-1 & \ddots &\ddots& \\
						& \ddots&\ddots& -1\\
						0&&-1&2
					\end{bmatrix}
					\in \mathbb{R}^{d\times d}.
					$$
				\end{enumerate}
				\noindent Take the mean of matrices $\mathbf{A} = \frac{1}{n}\sum_{i=1}^n\mathbf{A}_i$.\\
				\noindent Find the minimum eigenvalue $\lambda_{min}\(\mathbf{A}\)$.\\
				\noindent For $i = 1 \ \text{to} \ n$:
				\begin{enumerate}
					\item Update matrix $\mathbf{A}_i = \mathbf{A}_i + \(\lambda - \lambda_{min}\(\mathbf{A}\)\)\mathbf{I}$.
				\end{enumerate}
				\noindent Take starting point $x^0 = (\sqrt{d}, 0,\ldots, 0)$.\\
				\noindent \textbf{Output}: matrices $\mathbf{A}_1,\ldots,\mathbf{A}_n$, vectors $b_1,\ldots,b_n$, starting point $x^0$.
			\end{minipage}
		\end{center}
		\caption{\textsc{Generate quadratic optimization task with controlled $L_\pm$} \citep{tyurin2022weightedab}}
		\label{alg:l_pm_generation}
	\end{algorithm}
	
	\begin{algorithm}[H]
		\begin{center}
			\begin{minipage}{\textwidth}
				\noindent\textbf{Input}: number nodes $n$, dimension $d$ and noise scale $s$.\newline
				\noindent For $i = 1 \ \text{to} \ n$:
				\begin{enumerate}
					\item Generate random noises $\nu_i^s=1 + s\xi_i^s$, i.i.d $\xi_i^s \sim \text{ExponentialDistribution}(1)$.\\
					\item Generate random noises $\nu_i^b=s\xi_i^b$, i.i.d $\xi_i^b \sim \text{NormalDistribution}(0,1)$.\\
					\item Take vector $b_i = \(-1+\nu_i^b, 0, \ldots,0\)\in \mathbb{R}^d$.
					\item Take the initial tridiagonal matrix
					$$
					\mathbf{A}_{i} = \frac{\nu_i^s}{4}
					\begin{bmatrix}
						2 &-1 & &0\\
						-1 & \ddots &\ddots& \\
						& \ddots&\ddots& -1\\
						0&&-1&2
					\end{bmatrix}
					\in \mathbb{R}^{d\times d}.
					$$
				\end{enumerate}
				\noindent Take starting point $x^0 = (\sqrt{d}, 0,\ldots, 0)$.\\
				\noindent \textbf{Output}: matrices $\mathbf{A}_1,\ldots,\mathbf{A}_n$, vectors $b_1,\ldots,b_n$, starting point $x^0$.
			\end{minipage}
		\end{center}
		\caption{\textsc{Generate quadratic optimization task with controlled $L_i$} \citep{tyurin2022weightedab}}
		\label{alg:l_i_generation}
	\end{algorithm}
	
	\section{AUXILIARY FACTS}
	\subsection{Proof of Proposition~\ref{proposition:zero_hessian_implies_homogeneity}}
	\begin{proof}
		Since all $f_i(x), i\in[n],$ are equal to $f(x),$ we have that 
		\begin{equation}\label{eq:zero_hessian_variance}
			\frac{1}{n}\sum_{i=1}^n\norm{\nabla f_i(x) - \nabla f_i(y)}^2 = \norm{\nabla f(x) - \nabla f(y)}^2, \quad x,y\in\mathbb{R}^d.
		\end{equation}
		It immediately implies that the Hessian variance $L_{\pm}^2$ is equal to zero. 
		
		In case when functions are identical up to some random linear perturbation, assume that, for $i\in[n],$ $f_i(x)=\varphi(x)+b_i^{\top}x+c_i,$ where $\varphi(x):\mathbb{R}^d\to\mathbb{R}$ is a differentiable function, $b_i\in\mathbb{R}^d,$ $c_i\in\mathbb{R}.$ Observe that in this case~\eqref{eq:zero_hessian_variance} also holds true, and, therefore, $L_{\pm}^2=0.$
	\end{proof}

\end{document}